\documentclass[11pt]{article}
\usepackage[margin=1in]{geometry}
\usepackage{times}
\usepackage{graphicx}
\usepackage{amsthm, amsmath, amssymb}
\usepackage{natbib}
\usepackage{authblk}
\usepackage{hyperref}

\AtBeginDocument{%
  }

\usepackage{graphicx}
\usepackage{adjustbox}

\usepackage{amsthm}
\newtheorem{assumption}{Assumption}
\newtheorem{theorem}{Theorem}
\newtheorem{lemma}[theorem]{Lemma}

\usepackage{algorithm}
\usepackage{algpseudocode}
\usepackage{url}

\usepackage{afterpage} % <=== figure float
\usepackage{placeins}  % <=== figure float
\usepackage{flafter} % <=== figure float
\usepackage{neuralnetwork}
\usetikzlibrary{shapes.geometric, arrows.meta, positioning}

\usepackage{authblk}

\author{
    Tomu Hirata\footnote{Databricks Japan, Inc, Tokyo, Japan 
(\texttt{tomu.hirata@databricks.com})
}, \ \
    Undral Byambadalai\footnote{
    CyberAgent, Inc, Tokyo, Japan
    ( \texttt{undral\_byambadalai@cyberagent.co.jp}
    )
    }, \ \
    Tatsushi Oka\footnote{
    Department of Economics, Keio University, Tokyo, Japan 
    (\texttt{tatsushi.oka@keio.jp})
    },  \ \
    Shota Yasui\footnote{
    CyberAgent, Inc, Tokyo, Japan
    (\texttt{yasui\_shota@cyberagent.co.jp})
    }, \ \ 
    Shingo Uto\footnote{
    AbemaTV, Tokyo, Japan
    (\texttt{uto\_shingo@abema.tv})
    } 
}

\begin{document}

\title{Efficient and Scalable Estimation of Distributional Treatment Effects with Multi-Task Neural Networks \\ \vspace{0.5cm}}

\maketitle

\begin{abstract}
We propose a novel multi-task neural network approach for estimating distributional treatment effects (DTE) in randomized experiments. While DTE provides more granular insights into the experiment outcomes over conventional methods focusing on the Average Treatment Effect (ATE), estimating it with regression adjustment methods presents significant challenges. Specifically, precision in the distribution tails
suffers due to data imbalance, and computational inefficiencies arise
from the need to solve numerous regression problems, particularly
in large-scale datasets commonly encountered in industry. To address these limitations, our method leverages multi-task neural networks to estimate conditional outcome distributions while incorporating monotonic shape constraints and multi-threshold label learning to enhance accuracy. To demonstrate the practical effectiveness of our proposed method, we apply our method to both simulated and real-world datasets, including a randomized field experiment aimed at reducing water consumption in the US and a large-scale A/B test from a leading streaming platform in Japan. \footnote{The proposed algorithm is available in the Python library \texttt{dte-adj} (\href{https://pypi.org/project/dte-adj/}{https://pypi.org/project/dte-adj/})}
The experimental results consistently demonstrate superior performance across various datasets, establishing our method as a robust and practical solution for modern causal inference applications requiring a detailed understanding of treatment effect heterogeneity.

\vspace{0.7cm}
\end{abstract}

\nopagebreak

\newpage
\section{Introduction}
Randomized experiments, also known as A/B tests, have served as a cornerstone for understanding intervention effects and shaping policy decisions since Fisher's seminal work \cite{Fisher1937}. The framework of causal inference through randomized experiments has become foundational across diverse scientific disciplines \cite{Rubin1980, Imbens2015, Rubin1974} and has been widely adopted by technology companies as a standard practice in their decision-making processes 
\cite{tang2010overlapping, xie2016improving, Kohavi2020}.

While the Average Treatment Effect (ATE) is the conventional measure in randomized experiments, it is subject to two significant limitations: imprecise estimates with high variance relative to their magnitudes \cite{lewis2015unfavorable}, and incomplete characterization of intervention impacts across the outcome distribution. Although numerous studies have proposed methods for capturing distributional impacts of treatments \cite{Abadie2002, Firpo2007, ROTHE201056, Athey2006, FORTIN20111,  Wang03072018, Chernozhukov2005}, addressing both estimation precision and distributional characterization remains a challenge.

% This paper: regression adjustment and DTE
In this paper, we propose a novel approach to estimating Distributional Treatment Effects (DTE) using multi-task neural networks (NN). The DTE captures comprehensive treatment effects by measuring changes across the outcome distribution induced by treatments. While DTE estimation can be performed by directly comparing empirical outcome distributions between treatment and control groups, this approach often overlooks valuable pre-treatment auxiliary information commonly available in randomized experiments. Our method extends existing regression-adjustment techniques that leverage pre-treatment covariates for ATE estimation, an approach widely adopted in industry to enhance experimental sensitivity \cite{Deng2013, Xie2016}. By incorporating regression-adjustment methods within a neural network architecture for distributional regression, we develop a more powerful framework for DTE estimation.

% advantages 
Our proposed method enhances estimation precision, particularly in the distribution tails, through a multi-task neural network architecture with shape restrictions that jointly learns from multiple parts of the outcome distribution. By estimating distributional features simultaneously rather than solving separate regression problems, our method achieves significant computational efficiency. These methodological advances prove particularly valuable for analyzing large-scale datasets commonly encountered in industrial applications, as demonstrated in a US water conservation experiment and a large-scale A/B test on a Japanese streaming platform.

Our primary contributions are summarized as follows:

\vspace{-0.01cm}
\begin{itemize}
\item We propose multi-task neural networks for estimating conditional outcome distributions, substantially improving computational efficiency.
For illustration, see Figure \ref{fig:concept}.
\item 
We improve the estimation precision of distributional treatment effects by incorporating shape constraints into the neural network.

\item 
%Demonstrated the effectiveness of the proposed method through simulation and real-world scenarios
We validate the efficacy of our proposed methodology through extensive simulation studies and empirical analyses of two real-world datasets.
\end{itemize}

The rest of this paper is structured as follows. 
Section \ref{sec:related-work}
describes related research. We explain the problem setup in Section \ref{sec:setup} and the proposed methodology in Section \ref{sec:proposed-method}. 
Section \ref{sec:result} presents simulation studies and two real-world applications.
We conclude the paper in Section \ref{sec:conclusion}. The Appendix in the paper includes additional details and results.

\begin{figure}
    \centering
    \includegraphics[width=0.5\linewidth]{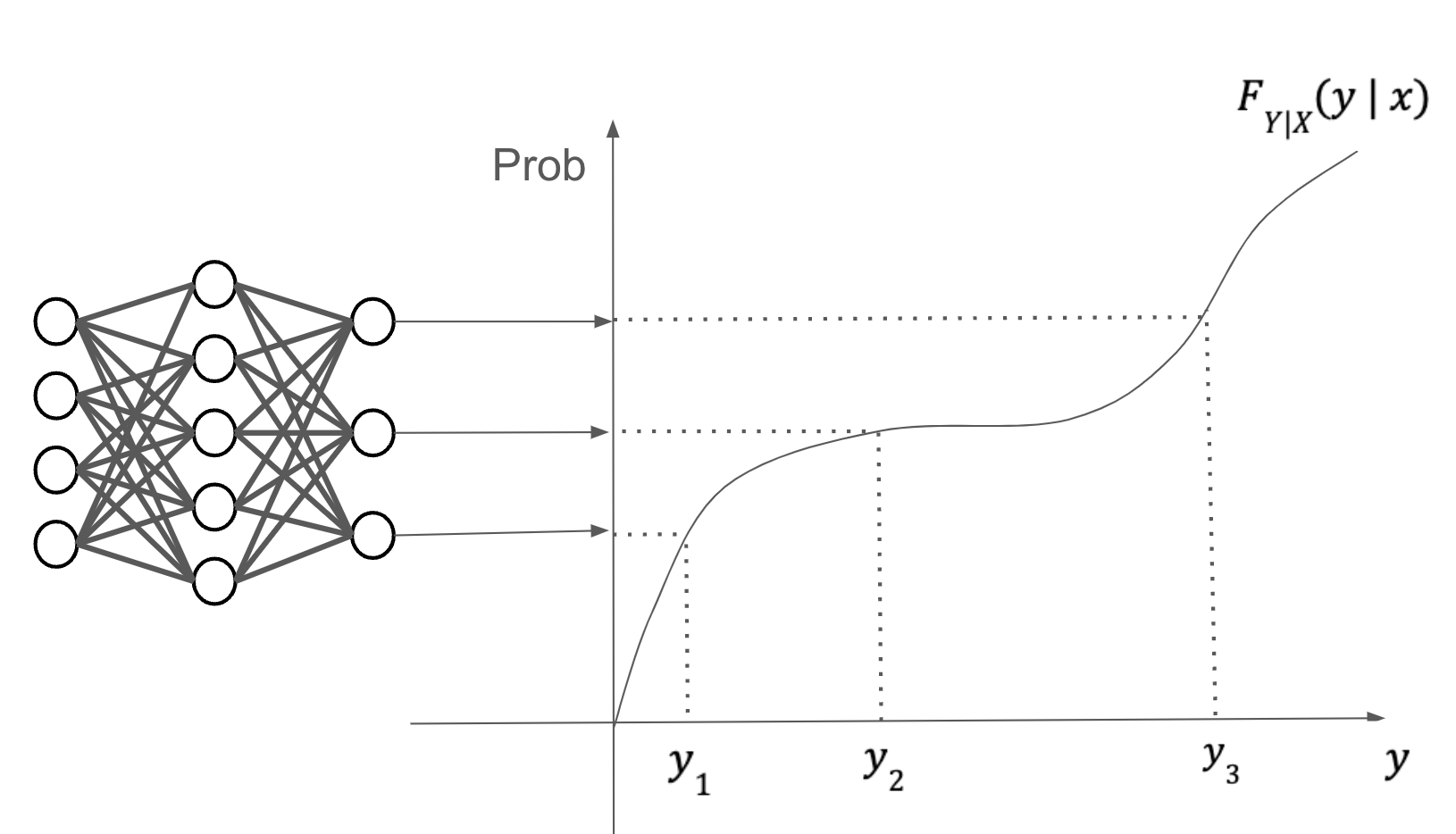}
    \caption{Illustration: Multi-Task neural network estimating the conditional distribution function $F_{Y|X}(\cdot|x)$.}
    \label{fig:concept}
\end{figure}

\section{Related Work}
\label{sec:related-work}
\textbf{Distributional Treatment Effect}:
Although the ATE remains a standard metric in treatment evaluation, the distributional characterization of treatment effects provides more comprehensive insights into treatment heterogeneity. This has motivated extensive research on both Quantile Treatment Effects (QTE) and Distributional Treatment Effects (DTE) \cite{Abadie2002, Chernozhukov2005, Athey2006, Callaway2019, Chernozhukov2013, Chernozhukov2020, kallus2024localized, 11fde4a6-7813-3966-82c1-2aeb85f068b8, callaway2018quantile, xu2018}. Unlike QTE which requires continuous outcome distributions, DTE can be applied to both discrete and continuous outcomes.

\vspace{0.2cm}
\noindent \textbf{Regression Adjustment}:
When estimating the treatment effect of RCT, the sensitivity of the estimation is critical. For example, failure to detect an effect due to a large variance leads to potential revenue loss in industry. Regression adjustment is a method that utilizes pre-treatment covariates to estimate the treatment effect more precisely. Various researches have explored the usage of regression adjustment to gain precise estimation for the ATE including 
\cite{Freedman2006, Berk2013, Lin2013, Rosenblum2010, Tsiatis2008, Imani2018, Rosenbaum2002, Yang2001}.
Despite the challenges in learning a correct relationship between pre-treatment covariates and outcome variables, it has been proved that the regression-adjusted estimation is unbiased, consistent, and practically useful.
Recent studies have extended regression adjustment beyond ATE to distributional treatment effects, establishing theoretical foundations for its application \cite{jiang2023regression, oka2024, Byambadalai2024, byambadalai2025efficient}. Building on this line of work, we enhance regression adjustment for distributional treatment effects by improving the underlying algorithm and integrating recent advancements in deep learning.

\vspace{0.2cm}
\noindent \textbf{Multi-Task Learning}:
Multi-task learning enables a single model to learn multiple related tasks simultaneously, leveraging shared representations to enhance performance \cite{caruana1997multitask,zhang2018overview, ruder2017overviewmultitasklearning,Michael2020}. This joint learning strategy improves generalization and mitigates overfitting \cite{zhang2021survey}. The framework proves particularly valuable in multi-label classification, where outputs can belong to multiple categories simultaneously \cite{tarekegn2024deeplearningmultilabellearning, zhang2014reviewmlc}. By treating each label prediction as a distinct yet related task, multi-task learning effectively captures label dependencies \cite{Min2006, dembczynski2012label, su2022zlprnovellossmultilabel} and leverages shared patterns within a unified framework \cite{argyriou2007multitaskfeature}. 
The deep neural network approach has been studied to estimate distributional regression based on binary cross-entropy \cite{li2019deepdistributionregression}
and quantile regression \cite{tagasovska2019singlemodeluncertainties,white1992nonparametric}. 
Our approach reformulates distributional regression as a multi-label classification task, utilizing multi-task learning to share representations across the distribution.

\vspace{0.2cm}
\noindent \textbf{Shape Restriction}:
Shape restriction on statistical models has a long history in statistics and econometrics, and there is a broad range of research about various shape restrictions such as concavity or monotonicity for parametric or semiparametric statistical models \cite{Barlow1972,Hastie1990,Matzkin1994, Horowitz2017,Chetverikov2018,Dümbgen2024}.
These shape restrictions are used for stabilizing the semiparametric models or increasing the efficiency of the estimation. They work particularly well if simple estimation is difficult due to a small sample size. 
Recent studies explored the shape restriction on machine learning (ML) models with higher dimensions of covariates \cite{gupta2016monotonic,You2017, gupta2018diminishing, cotter2019shape}, ensuring the monotonic relationship between input variables and outputs of neural networks. Wu et al. \cite{dnet} incorporated a cumsum function to ensure the monotonicity of individualized treatment effects (ITEs) estimation across different quantiles.
In this study, we introduce a monotonic constraint on the output of neural networks to enhance the stability and precision of unconditional treatment effects.

\section{Problem Setup}
\label{sec:setup}
\subsection{Setup and Notation}

% population 
We consider a randomized experiment with \( K \) treatment arms, where each unit is assigned to treatment \( w \) with probability \( \pi_w : = P(W = w) \) such that \( \sum_{w \in \mathcal{W}} \pi_w = 1 \). The observed data consist of \( n \) independent samples \( \{(X_i, W_i, Y_i)\}_{i=1}^n \) from \((X, W, Y) \), where \( X \in \mathcal{X} \subseteq  \mathbb{R}^{d_x} \) is a vector of pre-treatment covariates, \( W \in \mathcal{W} := \{1, \dots, K\} \) is the treatment assignment, and \( Y \in  \mathcal{Y} \subseteq \mathbb{R} \) is the observed outcome.  Under the potential outcomes framework \cite{Rubin1974, Rubin1980, Imbens2015}, each unit has potential outcomes \( Y(1), \dots, Y(K) \), but only the outcome corresponding to the assigned treatment is observed, denoted as \( Y = Y(W) \). Each treatment group contains \( n_w \) observations, satisfying \( \sum_{w \in \mathcal{W}} n_w = n \).

\vspace{0.25cm}
Throughout the paper, we maintain the following assumptions:
\begin{assumption}
\label{assumption:independence}
    $Y(1), \dots, Y(K), X \perp\!\!\!\perp W$.
\end{assumption}

\begin{assumption}
\label{assumption:positivity}
    $0 < \pi_w < 1, \quad \forall w \in \mathcal{W}$.
\end{assumption}
\noindent 
Assumption \ref{assumption:independence} ensures treatment assignment is independent of potential outcomes and pre-treatment covariates, while Assumption \ref{assumption:positivity} guarantees non-zero assignment probabilities. These assumptions are valid under randomized experiments.

We define the cumulative distribution functions (CDFs) of potential outcomes as:
\[
F_{Y(w)}(y) := \Pr \big ( Y(w)  \le y \big), 
%= \int_{-\infty}^y f_{Y(w)}(t) \, dt.
\]
for $y \in \mathcal{Y} \subseteq \mathbb{R}$ and $w \in \mathcal{W}$. 
While potential outcomes are unobservable, the potential outcome distribution $F_{Y(w)}(y)$ becomes identifiable when Assumptions \ref{assumption:independence} and \ref{assumption:positivity} are satisfied, as it equals the observed outcome distribution $F_{Y|W}(y|w) := \Pr(Y \leq y \mid W = w)$ for each treatment $w$.
In practice, we usually have a set of scalar-valued locations $\tilde{\mathcal{Y}} \subseteq \mathcal{Y}$ for which cumulative distribution functions are estimated and the following set of parameters are computed together:
\[
 \mathcal{F}_{Y(w)} := \big \{F_{Y(w)}(y) \mid y \in \tilde{\mathcal{Y}} \big\}.
\]
This collection of CDFs provides a distributional characterization of the treatment effect across the outcome distribution. 

\subsection{Distributional Treatment Effect}
We formally define the \emph{Distributional Treatment Effect} (DTE) between treatments $w, w' \in \mathcal{W}$ as
\begin{align*}
\Delta^{DTE}_{w, w'}(y) := F_{Y(w)}(y) - F_{Y(w')}(y),
\end{align*}
for $y\in \mathcal{Y}$. The DTE measures the difference between CDFs of the potential outcomes under the two treatments.

While our primary focus is the DTE, our framework extends to any parameters obtained through transformations of CDFs. An important example is the \emph{Probability Treatment Effect} (PTE) between treatments $w, w'\in \mathcal{W}$, defined as
\begin{align*}
\Delta^{PTE}_{w, w'}(y_j) := f_{Y(w)}(y_j) - f_{Y(w')}(y_j),
\end{align*}
where
\begin{align*}
f_{Y(w)}(y_j):&= Pr(y_{j-1} < Y(w) \leq y_j) \\
& = F_{Y(w)}(y_j) - F_{Y(w)}(y_{j-1}),
\end{align*}
for $y_j\in \mathcal{Y}$. Each quantity $f_{Y(w)}(y_j)$, which we refer to as the \emph{interval probability}, represents the difference between consecutive CDF values at $y_j$ and $y_{j-1}$. The PTE measures the differences in these interval probabilities between treatments.

\section{Estimation Methods}
\label{sec:proposed-method}
In this section, we describe our proposed method and its advantages. As a prerequisite, the existing algorithm is explained below.

\subsection{Single-Task Regression Adjustment Algorithm and its Limitations}

% algorithm 
We describe the existing regression-adjusted estimation procedure for distribution functions, which is outlined in Algorithm \ref{alg:regression_adjusted}.
The procedure begins with estimating the conditional distribution function of $Y(w)$ given $X$, defined as
\begin{equation*}
\gamma_y^{(w)}(X):=
\Pr\big ( Y(w) \le y|X \big),
\end{equation*}
and let $\widehat\gamma_y^{(w)}(X)$ denote its estimator 
for each $y \in \mathcal{Y}$ and $w \in \mathcal{W}$.
The estimation can be reformulated as a prediction problem by noting that $\gamma_y^{(w)}(X)= E[\mathbf{1}_{\{Y(w)\leq y\}}|X]$, wherein we predict binary outcomes $\mathbf{1}_{\{Y(w)\leq y\}}$ using pre-treatment covariates $X$ through a supervised learning algorithm $\mathcal S$ (e.g., linear regression, gradient boosting, random forests, single-task neural network). 
The estimation procedure relies on a sample-splitting method called cross-fitting, which, combined with the Neyman orthogonality, enables us to derive the asymptotic distribution of the regression-adjusted estimator under mild conditions \cite{chernozhukov2018debiased}.
 
Then, the regression-adjusted estimator  of $F_{Y(w)}(y)$ is computed as follows: 
\begin{equation}
\begin{split}
    \widehat{F}_{Y(w)}(y) = \frac{1}{n_w} \sum_{i:W_i = w} \big(\mathbf{1}_{\{Y_i \leq y\}} - \widehat{\gamma}_y^{(w)}(X_i)\big) 
    + \frac{1}{n} \sum_{i=1}^n \widehat{\gamma}_y^{(w)}(X_i).
    \label{math:regression-adjustment}
\end{split}
\end{equation}
This estimator modifies the empirical distribution function by subtracting the treatment group average of $\widehat{\gamma}_y^{(w)}(\cdot)$ and adding its overall average. The resulting estimator is unbiased since the expected values of these adjustment terms cancel out.

\begin{algorithm}[ht]
\caption{Single-Task Regression-Adjustment Algorithm}
\label{alg:regression_adjusted}

\textbf{Input:} Data $\{(X_i, W_i, Y_i)\}_{i=1}^n$ split randomly into $L$ roughly equal-sized folds ($L > 1$); a supervised learning algorithm $\mathcal S$.

\begin{algorithmic}[1]
\For{treatment group $w \in \mathcal{W}$}
    \For{level $y \in \tilde{\mathcal{Y}}$}
        \For{fold $\ell = 1$ to $L$}
            \State Train $\mathcal S$ on data excluding fold $\ell$, \;
            \State using observations in treatment group $w$.\;
            \State Use $\mathcal S$ to predict $\widehat{\gamma}_y^{(w)}(X_i)$ for each $X_i$ in fold $\ell$.\;
        \EndFor
        \State Compute $\widehat{F}_{Y(w)}(y)$ according to equation \eqref{math:regression-adjustment}.
    \EndFor
\EndFor
\end{algorithmic}
\end{algorithm}

Moreover, the following theorem holds for the adjusted estimator in equation \eqref{math:regression-adjustment} highlighting the efficiency gain through the regression adjustment compared with the empirical estimation defined by 
$\widehat{\mathbf{F}}^{empirical}_{Y(w)}(y) := \sum_{i: W_i=w} \mathbf{1}_{\{Y_i \leq y\}}/n_w.$
Consider the scenario where the true conditional distribution function, $\gamma_{y}^{(w)}(\cdot)$, is employed in \eqref{math:regression-adjustment}, leading to the idealized form of the regression-adjusted estimator, denoted by $ \widetilde{\theta}_{y}^{(w)}:= \widetilde{F}_{Y(w)}(y)$. Let the vector version be denoted by $\widetilde{\theta}_{y}:= (\widetilde{\theta}_{y}^{(1)}, \dots,\widetilde{\theta}_{y}^{(K)})^\top$. 

\vspace{0.15cm}
\begin{theorem}
    \label{theorem:efficiency}
    Suppose that $n_{w}/n = \pi_{w} + o(1)$ as $n\to \infty$ for every $w \in \mathcal{}{W}$. Then, the following inequalities hold: \\
    (a)
        for any $w \in \mathcal{W}$ and $y \in \mathcal{Y}$, 
        \begin{align*}
        \text{Var} \big(
        \widehat{\mathbf{F}}_{Y(w)}^{empirical}(y)
        \big)
        \ge
        \text{Var} \big(
          \widetilde{\mathbf{F}}_{Y(w)}(y)
        \big) + o(n^{-1}), 
      \end{align*}
       where the equality holds only if $F_{Y(w)|X}(y) = F_{Y(w)}(y)$, \\   
    (b) for any $y \in \mathcal{Y}$,
        \begin{align*}
            \text{Var}
            \big(
            \widehat{\theta}_{y}^{empirical}
            \big)
            \succeq
            \text{Var}
            \big(
            \widetilde{\theta}_{y}
            \big) + o(n^{-1}),
        \end{align*}
        where $\succeq$ denotes the positive semi-definiteness. 
        When
        $
               \text{Var} 
     \big (
         F_{Y(w)}(y|X)
         -
         r \cdot 
         F_{Y(w')}(y|X)
         \big )
         > 0 
        $
 for any distinct
$w, w' \in \mathcal{W}$ and $r \in \mathbb{R}$, 
        the positive definite result holds.   

\end{theorem}
\vspace{0.05cm}

Theorem \ref{theorem:efficiency}(a) states the variance reduction accomplished by applying regression adjustment to distribution functions. Additionally, Theorem \ref{theorem:efficiency}(b) elaborates on this variance reduction as a vector of regression-adjusted estimators, indicating the efficiency gain of the adjusted estimator for the DTE as a special case.

Despite effectively leveraging semiparametric models in conjunction with cross-fitting \cite{chernozhukov2018debiased, chernozhukov2022locally}, the algorithm has several limitations.
First, the algorithm has a high computational cost due to the need for iterations of single-target regression problems.
The computational complexity of the algorithm is $O(M \cdot D \cdot L \cdot P)$, where $M$ indicates the number of locations, $D$ indicates the number of treatment levels, $L$ indicates the number of folds, and $P := Cost(\mathcal S(1))$ indicates the computational cost of one iteration of the supervised statistical model training $\mathcal S$ for one target variable. When the number of locations is large, this algorithm takes a long time to compute the adjusted distributional parameters due to many iterations of ML model training.
Moreover, estimating the conditional distribution function is particularly challenging in the distribution tails due to label imbalance. For instance, when estimating the cumulative probability at the 90th percentile, 90\% of the training data satisfies $\mathbf{1}_{\{Y \leq y\}} = 1$, leading to difficulties in handling the imbalance. A similar issue arises at lower percentiles, where the majority of labels take the opposite value.

\subsection{Multi-Task Regression Adjustment
Algorithm}

\begin{figure*}
    \centering
    \includegraphics[width=0.56\linewidth]{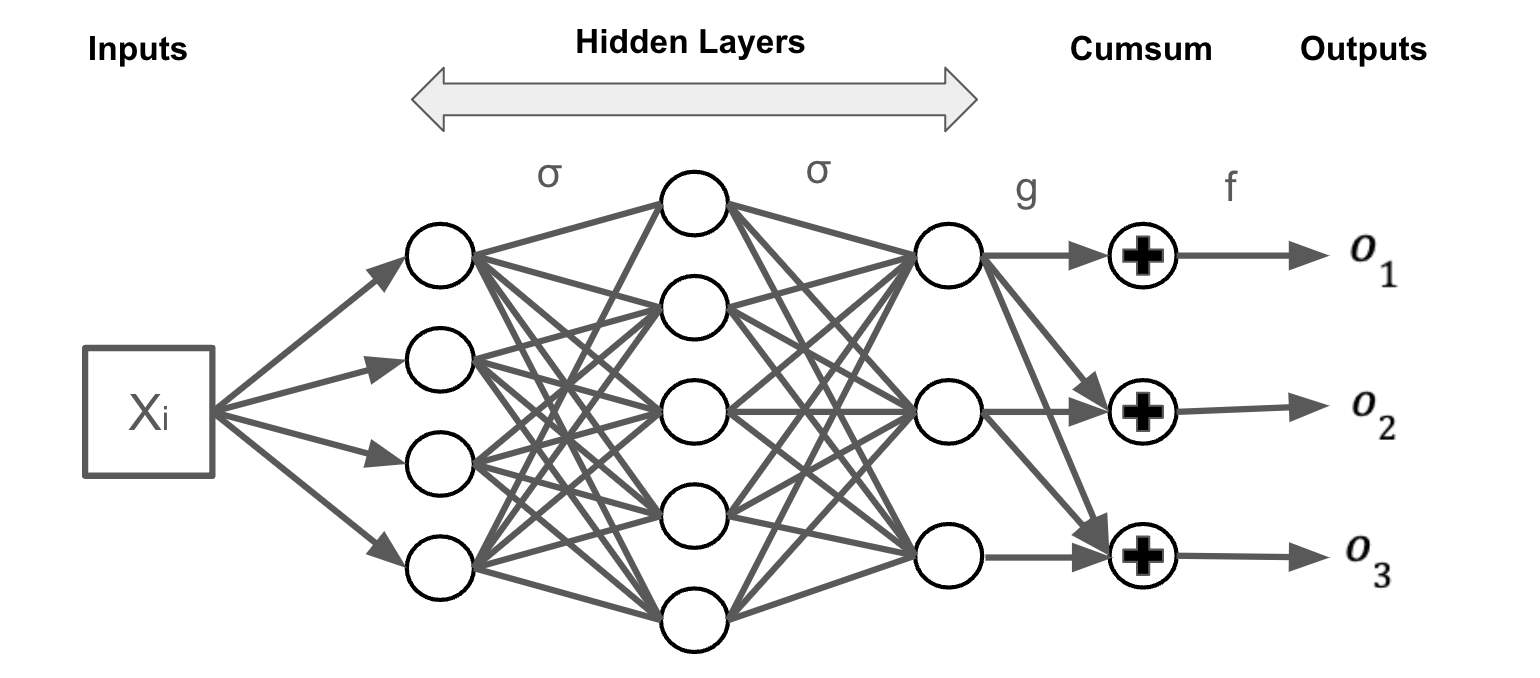}
    \caption{Illustration of the architecture of the proposed neural network model. The model accumulates the output of the hidden layers to have the monotonicity of the final outputs.}
    \label{fig:arch_model}
\end{figure*}

In this study, we reformulate the distribution regression into a multi-label classification task and propose a novel algorithm to estimate the regression-adjusted distribution function by leveraging a multi-task statistical model for estimating conditional distribution function $\widehat{\gamma}_y^{(w)}(X)$ for all $y \in \tilde{\mathcal{Y}}$ simultaneously. The procedure estimates a set of conditional distribution functions, defined as
\begin{equation*}
\mathbf{{\Gamma}}^{(w)} := 
\big \{{\gamma}_{y}^{(w)}(\cdot) \mid y \in \tilde{\mathcal{Y}} \big \}. 
\end{equation*}

\noindent
and let $\mathbf{\widehat{\Gamma}}^{(w)}$ be its estimator for each $w \in \mathcal{W}$ instead of estimating $\widehat{\gamma}_{y}^{(w)}(X)$ for each $y\in\tilde{\mathcal{Y}}$ individually.

The proposed method aims to overcome the high computational cost of the existing algorithm based on the following Assumption \ref{assumption:computation}, where $\mathcal  S(M)$ indicates the multi-task supervised learning algorithm with the number of target variables equal to $M$. This assumption posits that the computational cost of training the machine learning model increases at a sublinear rate for the size of the target variables.
\begin{assumption}
\label{assumption:computation}
    $$
    \tilde{P} := Cost(\mathcal S(M)) < O(P \cdot M).
    $$
\end{assumption}

If a single multi-task ML model is trained for all locations together and the computation cost for the single model is $\tilde{P}$, the total computation cost of the entire estimation process is $O(D \cdot L \cdot \tilde{P})$, which is lower than the existing algorithm $O(M \cdot D \cdot L \cdot P)$ when Assumption \ref{assumption:computation} is met. Note that Assumption \ref{assumption:computation} is not strong as many statistical models that can estimate multiple target variables together satisfy the assumption, including Linear Regression, Lasso, and Neural Networks. Appendix \ref{appendix:sub-linear} presents the theoretical analysis of the sub-linear assumption for linear regression and the empirical results of computational complexity across multiple neural network models. However, some algorithms, such as gradient-boosting trees or random forests, do not fulfill this condition since these algorithms need one model per target and usually the One vs Rest strategy is adopted for multi-output learning.
The proposed algorithm for estimating the regression-adjusted distribution function with a multi-task ML model is available in Algorithm \ref{alg:regression_adjusted_proposed} where $\mathcal{\widehat{F}}_{Y(w)} := \{\widehat{F}_{Y(w)}(y) \mid y \in \tilde{\mathcal{Y}}\}$.

\vspace{0.25cm}
\begin{algorithm}[ht]
\caption{Multi Task Regression-Adjustment Algorithm}
\label{alg:regression_adjusted_proposed}
\textbf{Input:} Data $\{(X_i, W_i, Y_i)\}_{i=1}^n$ split randomly into $L$ roughly equal-sized folds ($L > 1$); a multi-task supervised learning algorithm $\mathcal S(M)$.

\begin{algorithmic}[1]
\For{treatment group $w \in \mathcal{W}$}
    \For{fold $\ell = 1$ to $L$}
        \State Train $\mathcal S(M)$ on data excluding fold $\ell$, \;
        \State using observations in treatment group $w$.\;
        \State Use $\mathcal S(M)$ to predict $\mathbf{\widehat{\Gamma}}^{(w)}$ for observations in fold $\ell$.\;
    \EndFor
    \State Compute $\mathcal{\widehat{F}}_{Y(w)}$ according to equation \eqref{math:regression-adjustment}.\;
\EndFor
\end{algorithmic}
\end{algorithm}

Beyond computational efficiency, the multi-task model can learn richer representations compared to a single model trained for a single location. By simultaneously estimating multiple locations, the multi-task approach helps mitigate the label imbalance issue discussed in the previous section. Since the underlying relationship between covariates and the outcome remains consistent across locations, sharing representations through multi-task learning may improve estimation precision, as suggested in previous research \cite{Michael2020}.

Moreover, using a single multi-task ML model for estimating cumulative probability for each $y \in \tilde{\mathcal{Y}}$ enables us to enforce a monotonicity constraint across estimated probabilities in multiple locations. 
Specifically, for all $w\in\mathcal W$, the true conditional distribution function satisfies the following:
\begin{equation}
    \gamma_{y_s}^{(w)}(X)  \geq \gamma_{y_t}^{(w)}(X)  , \ \
    \forall y_{s}, y_{t} \in \mathcal{Y} \text{ s.t. } y_{s} \geq y_{t} \text{ a.s. }
    \label{math:monotonic_true}
\end{equation}

Since the true function satisfies the monotonicity property above, enforcing this constraint not only enhances interpretability \cite{gupta2018diminishing,You2017}, but also has the potential to improve training efficiency, as suggested by previous research \cite{Chetverikov2018}.
Specifically, for all $w \in \mathcal{W}$ and $X_i\in\mathcal{X}$, the estimated probabilities must satisfy
\begin{equation}
    \widehat\gamma_{y_s}^{(w)}(X_i) \geq \widehat\gamma_{y_t}^{(w)}(X_i) , 
    \ \ 
    \forall y_{s}, y_{t} \in \tilde{\mathcal{Y}} \text{ s.t. } y_{s} \geq y_{t}.
    \label{math:monotonic_estimate}
\end{equation} 

This monotonicity constraint, which may not be ensured when training separate single-task ML models for each target variable, can be enforced through a multi-task ML model. In this study, the DTE regression adjustment model is implemented based on a fully connected neural network. The architecture is illustrated in Figure \ref{fig:arch_model}. In this figure, $\sigma:\mathcal{X} \rightarrow\mathbb{R}$ is an activation function (e.g., ReLU or Sigmoid) used within the hidden layers, $g:\mathbb{R}\rightarrow[0, \infty)$ is a monotonic non-negative function for the last hidden layer, and $f:[0,\infty) \rightarrow [0,1]$ represents a monotonic map to compute the output probabilities.

In our proposed model, the monotonic constraint of the outputs is accomplished by the combination of $g$ and $f$. Let $H$ denote the number of hidden layers, $h_{j}^{k}$ be the $j$th element of the hidden layer $k \in \{1, 2, \dots, H\}$, and $o_{j}:=\widehat\gamma_{y_j}^{(w)}(X)$ be the $j$th element of the final output. Also, suppose  $\tilde{\mathcal{Y}}$ is sorted before training the model. The layers up to the final hidden layer are standard fully-connected neural network layers. In the final hidden layer $H$, a vector of non-negative values is obtained by applying $g$ to its output. Then, a vector of non-negative values $\tilde{h}^{H} = \{ \tilde{h}_{1}^{H},\tilde{h}_{2}^{H}, \dots, \tilde{h}_{M}^{H} \}$ is computed where each component $\tilde{h}_{j}^{H}$ is obtained by accumulating the non-negative components up to that point from the previous layer as
\begin{align*}
\tilde{h}_{j}^{H} = \sum_{m=1}^{j}g(h_{m}^{H}) , \forall j\in\{1,2, \dots, M\}.
\end{align*}
Then, by construction, its elements satisfy the monotonicity property:
\begin{align*}
\tilde{h}_{s}^{H} \geq \tilde{h}_{t}^{H} ,\forall s, t \in \{1, 2, \dots, M\} \text{ s.t. } s \geq t.
\end{align*}

Finally, the final output is obtained by applying $f$ function to each component in the last hidden layer as $o_{j} =  f(\tilde{h}_{j}^{H})$ for $j\in\{1, \dots, M\}$, and satisfies the monotonicity property given in \eqref{math:monotonic_estimate} by applying the monotonic function $f$ to the monotonically increasing vector.

\section{Empirical Results}
In this section, we evaluate our proposed multi-task regression-adjusted estimators against other methods in three distinct settings: simulation studies, a water conservation experiment, and a content promotion A/B test at a Japanese video streaming platform.

\label{sec:result}
\subsection{Simulation}
We conduct a Monte Carlo simulation study to assess the finite-sample performance of our adjusted estimators. Specifically, we evaluate the effectiveness of the proposed multi-task neural networks model adjustment in variance reduction, comparing it with the empirical distribution function, linear regression adjustment and the single-task NN model adjustment. The details of the data generation process (DGP) are provided in Appendix~\ref{simulation:DGP1}.  

For each simulation iteration, we generate a sample of size \( n = 1000 \) following the specified DGP and compute the distributional treatment effects (DTE) at quantiles \( \{0.05, 0.10, \dots, 0.95\} \). This process is repeated \( S = 500 \) times. To establish a ground truth for evaluating estimation efficiency, we separately generate \( 10^5 \) samples from the same DGP and compute the corresponding DTE, which is used to calculate bias and mean squared error (MSE). The details of the multi-task NN model are described in Appendix~\ref{simulation:model}.  

Figure~\ref{fig:simulation_dte} illustrates the reduction rate of pointwise MSE for each method compared to the empirical method across different locations, where higher values indicate greater precision. The results indicate that the multi-task NN adjustment method, both with and without the monotonic constraint, reduces pointwise MSE by 50--65\%. This reduction is greater than that of the single-task NN adjustment and linear regression adjustment, which achieve reductions of 12--50\%. Notably, at the tails of the distribution, such as \( q = 0.05 \) and \( q = 0.95 \), the reduction in MSE is particularly pronounced.

Since all estimators used in the experiment are unbiased, there is no substantial difference in bias values. Therefore, the observed differences in MSE are primarily attributable to variance reduction.
These results indicate that multi-task learning significantly contributes to variance reduction. Moreover, since the monotonic multi-task NN adjustment method demonstrates the best performance, the monotonic restriction also contributes the variance reduction. 

%% MSE simulation 
\begin{figure}[ht]
  \centering
  \includegraphics[width=0.50\linewidth]{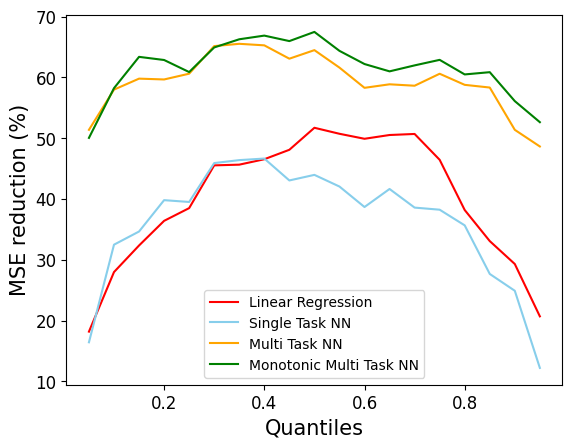}
  \caption{Pointwise MSE Reduction (\%) relative to the empirical DTE in the Simulation Study. DTE is computed for each quantile $q \in \{0.05, 0.1, \dots, 0.95\}$.}
  \label{fig:simulation_dte}
\end{figure}
% \vspace{0.1cm}

\begin{figure*}[!t]
  \centering
  \includegraphics[width=1.0\linewidth]{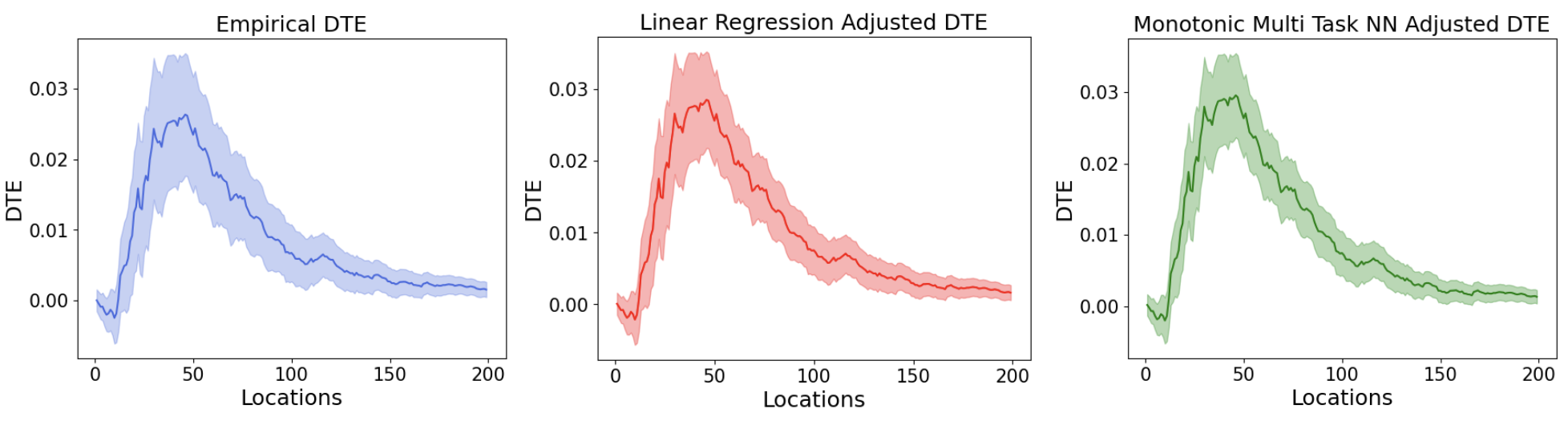}
  \caption{Distributional Treatment Effect (DTE) of a nudge (Strong Social Norm vs. Control) on water consumption (in thousands of gallons). 
  The left figure shows the empirical DTE, the middle figure shows the linear adjusted DTE, and the right figure shows the monotonic multi-task NN adjusted DTE, computed for $y \in \{1, 2, \dots, 200\}$.}\label{fig:water_consumption_multi_adjusted}
\end{figure*}

To evaluate the efficiency gains achieved through multi-task learning, we measured the average execution time of the estimation over 500 repetitions for each method. As presented in Table~\ref{table:simulation_time}, the average execution time of the multi-task NN adjustment is significantly lower than that of the single-task NN adjustment method. 

By leveraging the multi-task learning, the computation time for adjusted DTE estimation is reduced by around 80\% over the single-task NN adjustment under this simulation setting, highlighting the efficiency of the proposed approach. 

We also validate whether the improvement in the precision of distribution function estimation can be attributed to the enhancement of model prediction performance. To investigate this, we compare the prediction performance of linear regression, the single-task NN model, the multi-task NN model, and the monotonic multi-task NN model on a binary classification problem. We use the same data and locations of the simulation study and aggregate the prediction performance across all locations.

The prediction performance of each model obtained from 2-fold cross-validation is presented in Table~\ref{table:ablation_simulation_performance}. As shown in the numerical results, the multi-task and the monotonic multi-task NN models outperform other models. This result suggests that multi-task learning contributes to prediction performance at each location. Additionally, although the contribution is smaller compared to multi-task learning, the monotonic restriction also improves prediction performance. This result suggests that the improvement in precision for distribution function estimation is primarily driven by the superior predictive capability of the proposed monotonic multi-task NN model.

\begin{table*}[ht]
    \centering
    \begin{tabular}{|l|c|c|c|}
    \hline
    \textbf{Method} & 
    \textbf{(1) Simulation} & 
    \textbf{(2) Nudge} & \textbf{(3) ABEMA} \\
    \hline
    Linear Regression & 0.187 (sd 0.115) & 125 & 1,905 \\
    Single-Task NN & 7.89 (sd 0.609) & 1,888 & 8,249 \\
    Multi-Task NN & 1.71 (sd 0.126) & 124 & 2,087 \\
    Monotonic Multi-Task NN & 1.61 (sd 0.132) & 127 & 2,382 \\
    \hline
    \end{tabular}
    \caption{Execution Times in seconds for Different Methods Across Datasets. Rows represent methods and columns represent experimental datasets.}
    \label{table:simulation_time}
\end{table*}

\subsection{Nudges to Reduce Water Consumption}
In this section, we evaluate the performance of the proposed method using publicly available real randomized experiment data.

Specifically, we reanalyze data from a randomized experiment conducted by Ferraro and Price~\cite{Ferraro2013} in 2007, which examined the impact of norm-based messages, or ``nudges,'' on water consumption in Cobb County, Atlanta, Georgia. The experiment involved three distinct interventions designed to reduce water usage, each compared to a control group that received no nudge.

List et al.~\cite{List2024} re-estimated the ATE for the intervention labeled ``strong social norm,'' which incorporated both a prosocial appeal and social comparison, making it the most intensive intervention applied. Building on their analysis, we estimate the regression-adjusted DTE for this intervention relative to the control group, utilizing the same set of pre-treatment covariates, \(X\), representing monthly water consumption during the year preceding the experiment. Consequently, the covariate space \(X\) has a dimensionality of \(d_x = 12\). The outcome variable \(Y\) represents water consumption levels from June to September 2007, measured in thousands of gallons and treated as discretely distributed.

Figure~\ref{fig:water_consumption_multi_adjusted} presents the DTE estimation results using the empirical cumulative distribution function (CDF), linear regression adjustment, single-task NN adjustment, and the proposed multi-task NN adjustments. We calculate and compare standard errors (SE) using a multiplier bootstrap with \(B=5,000\) repetitions for each method as explained in Appendix \ref{app:multiplier-bootstrap}. Regardless of the method used, the results show an increase in the number of users with a consumption level of around 10-50 in the treatment group, while the number of users with higher consumption levels has decreased.

In real-world data, the true effect cannot be observed, so we cannot calculate MSE as in the previous section. Instead, we examine whether SE has been reduced. To examine the reduction in SE, we calculated the pointwise SE reduction relative to the empirical DTE for each location and presented it in Figure ~\ref{fig:dte_se_water_consumption}. 
The proposed method achieves a standard error reduction of 0--30\%, outperforming both the linear regression adjustment and the single-task NN adjustment. The methods using multi-task learning show a smaller difference in SE than single-task learning for the location between 0 and 100. On the other hand, for locations greater than 100, SE is more significantly reduced with multi-task learning. 
Predictions in the range of 100 to 200 are more challenging, as more than 95\% of records fall within the range of 0 to 100,
as illustrated in Figure~\ref{fig:water_consumption_hist} of Appendix \ref{ap:water_consumption}. 

The enhanced precision in this region suggests that learning labels jointly across all locations enables the model to better capture the underlying distribution and address the imbalanced data issue effectively. Though not significant, SE is reduced further around the tail of the distribution by the monotonic constraint.

\vspace{0.1cm}
\begin{figure}[ht]
    \centering
    \includegraphics[width=0.5\linewidth]{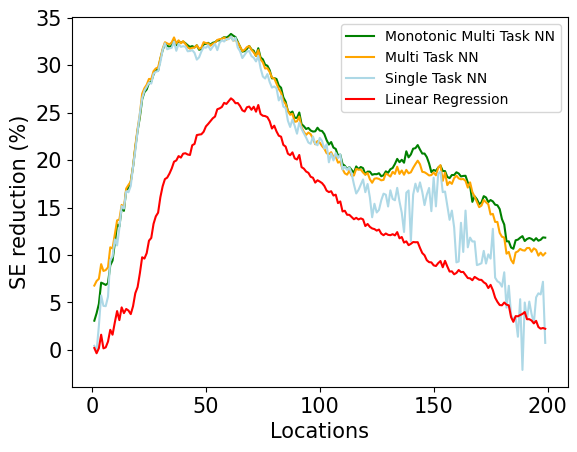}
    \caption{Pointwise SE reduction (\%) relative to the empirical DTE of a nudge 
    %(Strong Social Norm vs. Control) 
    on water consumption.}
    \label{fig:dte_se_water_consumption}
\end{figure}

\subsection{ABEMA Content Promotion}
Lastly, the proposed regression adjustment methods are applied to evaluate the result of an A/B test at ABEMA, a leading video streaming platform in Japan with over 30 million weekly active users. The A/B test was conducted to investigate the impact of top-of-screen content promotions in the ABEMA mobile application on viewer engagement over a four-week period.
In the experiment, users were randomly assigned to one of two treatment groups: those exposed to content promotions and those exposed to advertisements. Figure \ref{fig:abema-visual} illustrates these two treatment conditions with screen captures.

\begin{figure}[]
    \centering\includegraphics[width=0.5\linewidth]{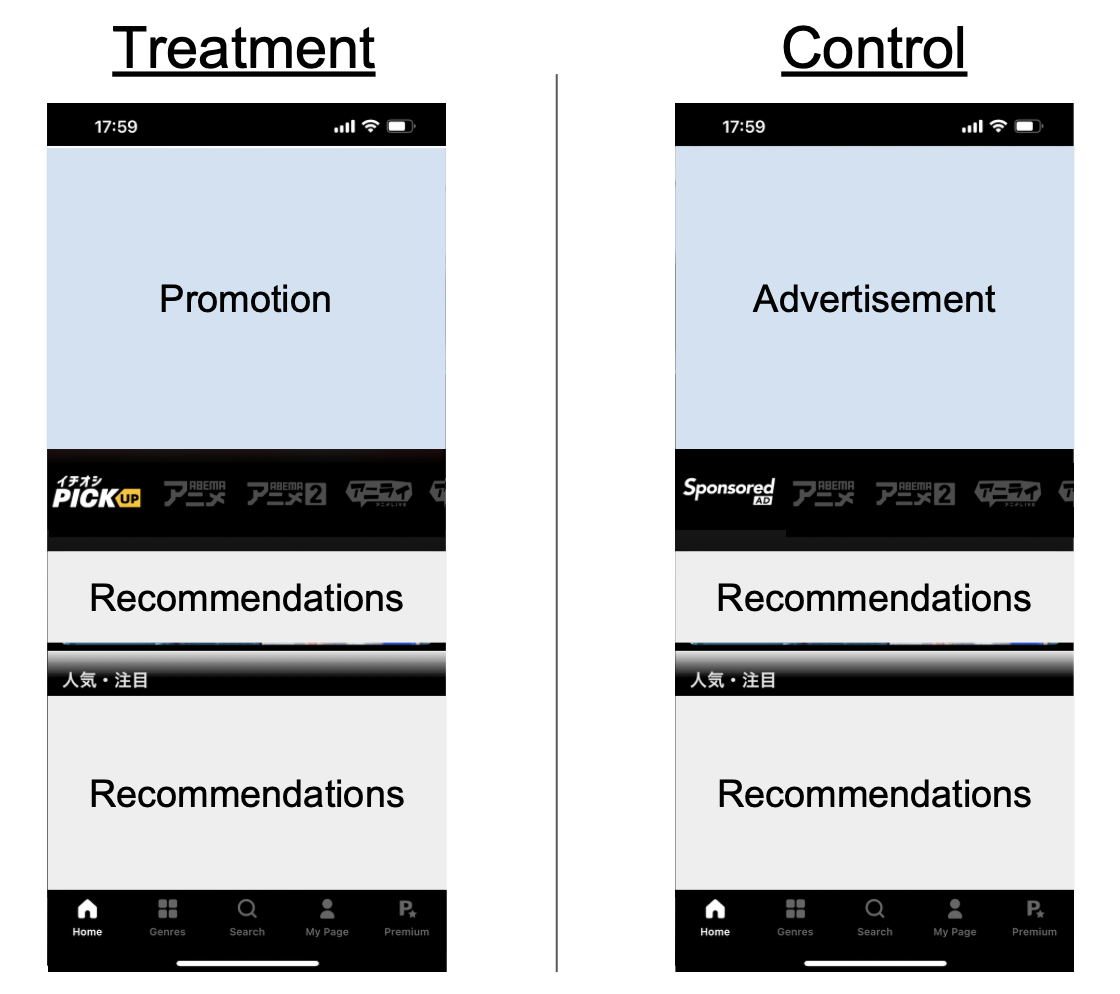}
    \caption{Top-of-Screen promotion in ABEMA application. Screen captures showing the two treatment conditions: content promotion (left) and advertisement (right).
    }
    \label{fig:abema-visual}
\end{figure}

The outcome variable in this experiment is the total viewing time (in minutes) for a selected series over the four-week experiment period. The pre-treatment covariates include age, gender, and various historical user activities recorded during the three weeks preceding the experiment. 
See Appendix~\ref{appendix:abema} for experiment settings.
For this analysis, we focus on a 46-minute comedy program. The experiment included a total of 4,311,905 users, which is a subset of total platform users.

The estimated ATE of the content promotion is 0.0145 (SE = 0.006, p = 0.01) minutes, corresponding to a 5.9\% increase in the average viewing time. This result indicates that the promotion successfully increased users' average viewing duration. However, an increase in average viewing time can arise from multiple scenarios, such as a rise in new users watching the content, a significant increase in users who watch for a short duration, or a growth in users who watch for a relatively long duration. The estimated average treatment effect does not provide particularly useful information in distinguishing between these scenarios. 

To examine the heterogeneous effects, we estimate the DTE. Figure~\ref{fig:abema-dte} presents the DTE estimation results using the empirical CDF, linear regression adjustment, and the proposed multi-task NN adjustments. Regardless of the method used, it is shown that there are fewer users in the treatment group at \( y = 0 \). This indicates that content promotion primarily encouraged users to start watching the series and reduced the number of users who did not engage with the content at all. Additionally, for \( y > 0 \), the number of users in the treatment group has increased, with a particularly larger increase among users who watched for a short duration. 
This suggests that while some users exposed to content promotion complete viewing, many others engage only briefly.

\begin{figure*}[]
    \centering
    \includegraphics[width=\linewidth]{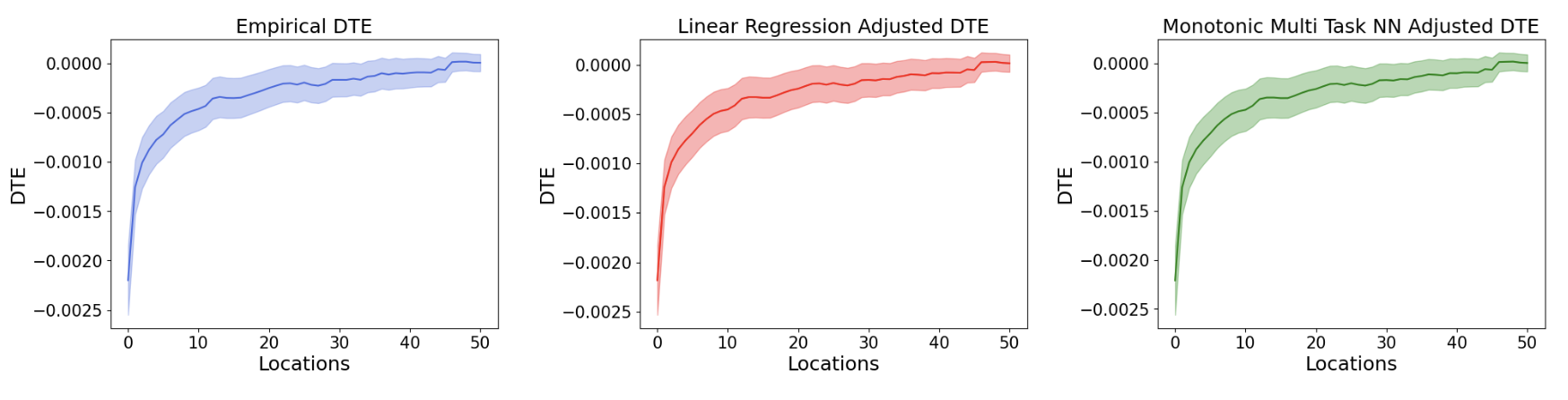}
    \caption{Distributional Treatment Effect (DTE) of the content promotion on ABEMA.
    The left figure shows the empirical DTE, the middle figure shows the linear adjusted DTE, and the right figure shows the monotonic multi-task NN adjusted DTE, computed for $y \in \{0, 1, 2, \dots, 50\}$.}
    \label{fig:abema-dte}
\end{figure*}

As illustrated in the SE reduction comparisons in Figure~\ref{fig:abema-dte-reduction}, all methods improve estimation precision relative to the empirical estimation. The SE reduction is relatively small at \( y = 0 \) but increases for \( y > 0 \), where the number of observations is lower. The proposed multi-task NN adjustment and monotonic multi-task NN adjustment achieve the highest SE reduction across most locations, reducing SE by 0--16\%. 

\begin{figure}[]
    \centering
    \includegraphics[width=0.5\linewidth]{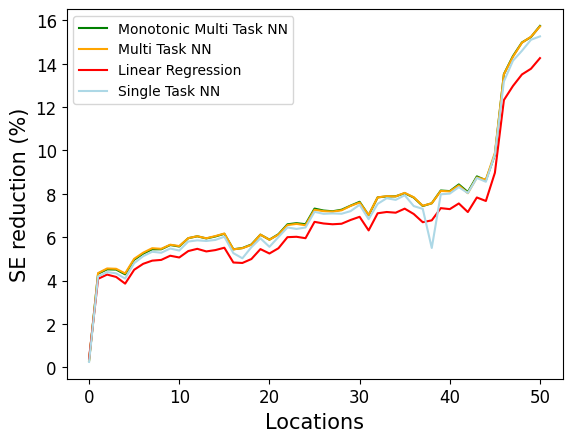}
    \caption{
    Pointwise SE reduction (\%) relative to empirical DTE of content promotion on 46-minute program viewing time.}
    \label{fig:abema-dte-reduction}
\end{figure}

Finally, we compare computational times across methods using two real-world datasets in Table~\ref{table:simulation_time}. Our multi-task neural network implementation demonstrates substantial computational advantages, yielding execution time reductions of 93\% and 71\% in the first and second experiments, respectively, compared to the single-task neural network implementation.

% \clearpage
\section{Conclusion}
\label{sec:conclusion}
In this study, we introduce a novel multi-task neural network approach with monotonic shape constraints for efficiently estimating distributional treatment effects in randomized experiments. Our method reduces computational costs drastically by leveraging multi-task learning, while enhancing precision through the monotonic shape constraint and multi-label learning. Theoretical analysis establishes the efficiency gains of our approach, and empirical results confirm significant reductions in execution time alongside improved estimation precision.

The method's computational advantages and shared representations are most pronounced when estimating cumulative distribution functions at multiple locations, though these benefits may be more modest for scenarios with fewer estimation points. Additionally, the computational efficiency gains depend on Assumption \ref{assumption:computation}, and may vary under different model architectures such as boosting trees. Nevertheless, our approach represents a significant advancement in the estimation of distributional treatment effects, offering a scalable and practical solution for modern large-scale randomized experiments. The framework we establish opens new possibilities for efficient treatment effect analysis across diverse experimental settings.

\clearpage
\bibliographystyle{plain}
\bibliography{base}

\begin{thebibliography}{10}

\bibitem{Abadie2002}
Alberto Abadie.
\newblock Bootstrap tests for distributional treatment effects in instrumental variable models.
\newblock {\em Journal of the American Statistical Association}, 97:284--292, 2002.

\bibitem{argyriou2007multitaskfeature}
Andreas Argyriou, Theodoros Evgeniou, and Massimiliano Pontil.
\newblock Multi-task feature learning.
\newblock {\em Advances in Neural Information Processing Systems}, 19:41--48, 2007.

\bibitem{Athey2006}
Susan Athey and Guido~W. Imbens.
\newblock Identification and inference in nonlinear difference-in-differences models.
\newblock {\em Econometrica}, 74:431--497, 2006.

\bibitem{Barlow1972}
Richard~E. Barlow.
\newblock {\em Statistical inference under order restrictions : the theory and application of isotonic regression}.
\newblock Wiley, 1972.

\bibitem{belloni2017program}
Alexandre Belloni, Victor Chernozhukov, Ivan Fernandez-Val, and Christian Hansen.
\newblock Program evaluation and causal inference with high-dimensional data.
\newblock {\em Econometrica}, 85(1):233--298, 2017.

\bibitem{Berk2013}
Richard Berk, Emil Pitkin, Lawrence Brown, Andreas Buja, Edward George, and Linda Zhao.
\newblock Covariance adjustments for the analysis of randomized field experiments.
\newblock {\em Evaluation review}, 37:170--196, 2013.

\bibitem{byambadalai2025efficient}
Undral Byambadalai, Tomu Hirata, Tatsushi Oka, and Shota Yasui.
\newblock On efficient estimation of distributional treatment effects under covariate-adaptive randomization.
\newblock {\em arXiv preprint arXiv:2506.05945}, 2025.

\bibitem{Byambadalai2024}
Undral Byambadalai, Tatsushi Oka, and Shota Yasui.
\newblock Estimating distributional treatment effects in randomized experiments: Machine learning for variance reduction.
\newblock 2024.

\bibitem{Callaway2019}
Brantly Callaway and Tong Li.
\newblock Quantile treatment effects in difference in differences models with panel data.
\newblock {\em Quantitative economics}, 10:1579--1618, 2019.

\bibitem{callaway2018quantile}
Brantly Callaway, Tong Li, and Tatsushi Oka.
\newblock Quantile treatment effects in difference in differences models under dependence restrictions and with only two time periods.
\newblock {\em Journal of Econometrics}, 206(2):395--413, 2018.

\bibitem{caruana1997multitask}
Rich Caruana.
\newblock Multitask learning.
\newblock {\em Machine Learning}, 28(1):41--75, 1997.

\bibitem{chernozhukov2018debiased}
Victor Chernozhukov, Denis Chetverikov, Mert Demirer, Esther Duflo, Christian Hansen, Whitney Newey, and James Robins.
\newblock {Double/debiased machine learning for treatment and structural parameters}.
\newblock {\em The Econometrics Journal}, 21(1):C1--C68, 01 2018.

\bibitem{chernozhukov2022locally}
Victor Chernozhukov, Juan~Carlos Escanciano, Hidehiko Ichimura, Whitney~K Newey, and James~M Robins.
\newblock Locally robust semiparametric estimation.
\newblock {\em Econometrica}, 90(4):1501--1535, 2022.

\bibitem{chernozhukov2013inference}
Victor Chernozhukov, Iv{\'a}n Fern{\'a}ndez-Val, and Blaise Melly.
\newblock Inference on counterfactual distributions.
\newblock {\em Econometrica}, 81(6):2205--2268, 2013.

\bibitem{Chernozhukov2013}
Victor Chernozhukov, Iván Fernández-Val, and Blaise Melly.
\newblock Inference on counterfactual distributions.
\newblock {\em Econometrica}, 81:2205--2268, 2013.

\bibitem{Chernozhukov2020}
Victor Chernozhukov, Iván Fernández-Val, Blaise Melly, and Kaspar Wüthrich.
\newblock Generic inference on quantile and quantile effect functions for discrete outcomes.
\newblock {\em Journal of the American Statistical Association}, 115:123--137, 2020.

\bibitem{Chernozhukov2005}
Victor Chernozhukov and Christian Hansen.
\newblock An iv model of quantile treatment effects.
\newblock {\em Econometrica}, 73:245--261, 2005.

\bibitem{Chetverikov2018}
Denis Chetverikov, Andres Santos, and Azeem~M Shaikh.
\newblock The econometrics of shape restrictions.
\newblock {\em Annual review of economics}, 10:31--63, 2018.

\bibitem{Chipman_2010}
Hugh~A. Chipman, Edward~I. George, and Robert~E. McCulloch.
\newblock Bart: Bayesian additive regression trees.
\newblock {\em The Annals of Applied Statistics}, 4(1), March 2010.

\bibitem{cotter2019shape}
Andrew Cotter, Maya Gupta, Heinrich Jiang, Erez Louidor, James Muller, Tamann Narayan, Serena Wang, and Tao Zhu.
\newblock Shape constraints for set functions.
\newblock In {\em International conference on machine learning}, pages 1388--1396. PMLR, 2019.

\bibitem{Michael2020}
Michael Crawshaw.
\newblock Multi-task learning with deep neural networks: {A} survey.
\newblock {\em CoRR}, abs/2009.09796, 2020.

\bibitem{dembczynski2012label}
Krzysztof Dembczy{\'n}ski, Willem Waegeman, Weiwei Cheng, and Eyke H{\"u}llermeier.
\newblock On label dependence and loss minimization in multi-label classification.
\newblock {\em Machine Learning}, 88:5--45, 2012.

\bibitem{Deng2013}
Alex Deng, Ya~Xu, Ron Kohavi, and Toby Walker.
\newblock Improving the sensitivity of online controlled experiments by utilizing pre-experiment data.
\newblock In {\em Proceedings of the Sixth ACM International Conference on Web Search and Data Mining}, pages 123--132. Association for Computing Machinery, 2013.

\bibitem{Dümbgen2024}
Lutz Dümbgen.
\newblock Shape-constrained statistical inference.
\newblock {\em Annual review of statistics and its application}, 11:373--391, 2024.

\bibitem{Ferraro2013}
Paul~J. Ferraro and Michael~K. Price.
\newblock Using nonpecuniary strategies to influence behavior: Evidence from a large-scale field experiment.
\newblock {\em The review of economics and statistics}, 95:64--73, 2013.

\bibitem{Firpo2007}
Sergio Firpo.
\newblock Efficient semiparametric estimation of quantile treatment effects.
\newblock {\em Econometrica}, 75(1):259--276, 2007.

\bibitem{Fisher1937}
Ronald~Aylmer Fisher.
\newblock {\em The design of experiments}.
\newblock Oliver and Boyd, 2nd ed. edition, 1937.

\bibitem{FORTIN20111}
Nicole Fortin, Thomas Lemieux, and Sergio Firpo.
\newblock Decomposition methods in economics.
\newblock 4:1--102, 2011.

\bibitem{Freedman2006}
David~A. Freedman.
\newblock Statistical models for causation: What inferential leverage do they provide?
\newblock {\em Evaluation review}, 30:691--713, 2006.

\bibitem{11fde4a6-7813-3966-82c1-2aeb85f068b8}
Markus Frölich and Blaise Melly.
\newblock Unconditional quantile treatment effects under endogeneity.
\newblock {\em Journal of Business \& Economic Statistics}, 31(3):346--357, 2013.

\bibitem{gupta2018diminishing}
Maya Gupta, Dara Bahri, Andrew Cotter, and Kevin Canini.
\newblock Diminishing returns shape constraints for interpretability and regularization.
\newblock {\em Advances in neural information processing systems}, 31, 2018.

\bibitem{gupta2016monotonic}
Maya Gupta, Andrew Cotter, Jan Pfeifer, Konstantin Voevodski, Kevin Canini, Alexander Mangylov, Wojciech Moczydlowski, and Alexander Van~Esbroeck.
\newblock Monotonic calibrated interpolated look-up tables.
\newblock {\em Journal of Machine Learning Research}, 17(109):1--47, 2016.

\bibitem{Hastie1990}
Trevor. Hastie.
\newblock {\em Generalized additive models}.
\newblock Chapman and Hall, 1st ed edition, 1990.

\bibitem{Horowitz2017}
Joel~L. Horowitz and Sokbae Lee.
\newblock Nonparametric estimation and inference under shape restrictions.
\newblock {\em Journal of econometrics}, 201:108--126, 2017.

\bibitem{Imani2018}
Ehsan Imani and Martha White.
\newblock Improving regression performance with distributional losses.
\newblock 2018.

\bibitem{Imbens2015}
Guido~W. Imbens and Donald~B. Rubin.
\newblock {\em Causal Inference for Statistics, Social, and Biomedical Sciences}.
\newblock Cambridge University Press, 2015.

\bibitem{jiang2023regression}
Liang Jiang, Peter~CB Phillips, Yubo Tao, and Yichong Zhang.
\newblock Regression-adjusted estimation of quantile treatment effects under covariate-adaptive randomizations.
\newblock {\em Journal of Econometrics}, 234(2):758--776, 2023.

\bibitem{kallus2024localized}
Nathan Kallus, Xiaojie Mao, and Masatoshi Uehara.
\newblock Localized debiased machine learning: Efficient inference on quantile treatment effects and beyond.
\newblock {\em Journal of Machine Learning Research}, 25(16):1--59, 2024.

\bibitem{Kohavi2020}
Ron Kohavi.
\newblock {\em Trustworthy online controlled experiments : a practical guide to A/B testing}.
\newblock Cambridge University Press, 2020.

\bibitem{Wang03072018}
Rui~Song Lan~Wang, Yu~Zhou and Ben Sherwood.
\newblock Quantile-optimal treatment regimes.
\newblock {\em Journal of the American Statistical Association}, 113(523):1243--1254, 2018.
\newblock PMID: 30416233.

\bibitem{lewis2015unfavorable}
Randall~A Lewis and Justin~M Rao.
\newblock The unfavorable economics of measuring the returns to advertising.
\newblock {\em The Quarterly Journal of Economics}, 130(4):1941--1973, 2015.

\bibitem{li2019deepdistributionregression}
Rui Li, Howard~D. Bondell, and Brian~J. Reich.
\newblock Deep distribution regression, 2019.

\bibitem{Lin2013}
Winston Lin.
\newblock Agnostic notes on regression adjustments to experimental data: Reexamining freedman's critique.
\newblock {\em The annals of applied statistics}, 7:295--318, 2013.

\bibitem{List2024}
John~A. List, Ian Muir, and Gregory Sun.
\newblock Using machine learning for efficient flexible regression adjustment in economic experiments.
\newblock {\em Econometric reviews}, pages 1--39, 2024.

\bibitem{Matzkin1994}
Rosa~L. Matzkin.
\newblock {\em Chapter 42 Restrictions of economic theory in nonparametric methods}, volume~4, pages 2523--2558.
\newblock Elsevier B.V, 1994.

\bibitem{oka2024}
Tatsushi Oka, Shota Yasui, Yuta Hayakawa, and Undral Byambadalai.
\newblock Regression adjustment for estimating distributional treatment effects in randomized controlled trials.
\newblock 2024.

\bibitem{paszke2019pytorchimperativestylehighperformance}
Adam Paszke, Sam Gross, Francisco Massa, Adam Lerer, James Bradbury, Gregory Chanan, Trevor Killeen, Zeming Lin, Natalia Gimelshein, Luca Antiga, Alban Desmaison, Andreas Köpf, Edward Yang, Zach DeVito, Martin Raison, Alykhan Tejani, Sasank Chilamkurthy, Benoit Steiner, Lu~Fang, Junjie Bai, and Soumith Chintala.
\newblock Pytorch: An imperative style, high-performance deep learning library, 2019.

\bibitem{Rosenbaum2002}
Paul~R. Rosenbaum.
\newblock Covariance adjustment in randomized experiments and observational studies.
\newblock {\em Statistical science}, 17:286--304, 2002.

\bibitem{Rosenblum2010}
Michael Rosenblum and Mark~J van~der Laan.
\newblock Simple, efficient estimators of treatment effects in randomized trials using generalized linear models to leverage baseline variables.
\newblock {\em The International Journal of Biostatistics}, 6:13--Article 13, 2010.

\bibitem{ROTHE201056}
Christoph Rothe.
\newblock Nonparametric estimation of distributional policy effects.
\newblock {\em Journal of Econometrics}, 155(1):56--70, 2010.

\bibitem{Rubin1974}
Donald~B Rubin.
\newblock Estimating causal effects of treatments in randomized and nonrandomized studies.
\newblock {\em Journal of educational psychology}, 66:688--701, 1974.

\bibitem{Rubin1980}
Donald~B Rubin.
\newblock Randomization analysis of experimental data: The fisher randomization test comment, 1980.

\bibitem{ruder2017overviewmultitasklearning}
Sebastian Ruder.
\newblock An overview of multi-task learning in deep neural networks.
\newblock arXiv preprint arXiv:1706.05098, 2017.

\bibitem{su2022zlprnovellossmultilabel}
Jianlin Su, Mingren Zhu, Ahmed Murtadha, Shengfeng Pan, Bo~Wen, and Yunfeng Liu.
\newblock Zlpr: A novel loss for multi-label classification, 2022.

\bibitem{tagasovska2019singlemodeluncertainties}
Nadezhda Tagasovska and David Lopez-Paz.
\newblock Single-model uncertainties for deep learning.
\newblock In {\em Advances in Neural Information Processing Systems 32}, pages 6417--6428, 2019.

\bibitem{tang2010overlapping}
Diane Tang, Ashish Agarwal, Deirdre O'Brien, and Mike Meyer.
\newblock Overlapping experiment infrastructure: More, better, faster experimentation.
\newblock In {\em Proceedings of the 16th ACM SIGKDD international conference on Knowledge discovery and data mining}, pages 17--26, 2010.

\bibitem{tarekegn2024deeplearningmultilabellearning}
Adane~Nega Tarekegn, Mohib Ullah, and Faouzi~Alaya Cheikh.
\newblock Deep learning for multi-label learning: A comprehensive survey, 2024.

\bibitem{Tsiatis2008}
Anastasios~A. Tsiatis, Marie Davidian, Min Zhang, and Xiaomin Lu.
\newblock Covariate adjustment for two-sample treatment comparisons in randomized clinical trials: A principled yet flexible approach.
\newblock {\em Statistics in medicine}, 27:4658--4677, 2008.

\bibitem{white1992nonparametric}
Halbert White.
\newblock Nonparametric estimation of conditional quantiles using neural networks.
\newblock In {\em Computing Science and Statistics: Statistics of Many Parameters: Curves, Images, Spatial Models}, pages 190--199. Springer, 1992.

\bibitem{dnet}
Guojun Wu, Ge~Song, Xiaoxiang Lv, Shikai Luo, Chengchun Shi, and Hongtu Zhu.
\newblock Dnet: Distributional network for distributional individualized treatment effects.
\newblock In {\em Proceedings of the 29th ACM SIGKDD Conference on Knowledge Discovery and Data Mining}, KDD '23, page 5215–5224, New York, NY, USA, 2023. Association for Computing Machinery.

\bibitem{xie2016improving}
Huizhi Xie and Juliette Aurisset.
\newblock Improving the sensitivity of online controlled experiments: Case studies at netflix.
\newblock In {\em Proceedings of the 22nd ACM SIGKDD International Conference on Knowledge Discovery and Data Mining}, pages 645--654, 2016.

\bibitem{Xie2016}
Huizhi Xie and Juliette Aurisset.
\newblock Improving the sensitivity of online controlled experiments: Case studies at netflix.
\newblock In {\em Proceedings of the 22nd ACM SIGKDD International Conference on Knowledge Discovery and Data Mining}, pages 645--654. ACM, 2016.

\bibitem{xu2018}
Dandan Xu, Michael~J. Daniels, and Almut~G. Winterstein.
\newblock A bayesian nonparametric approach to causal inference on quantiles.
\newblock {\em Biometrics}, 74(3):986--996, 2018.

\bibitem{Yang2001}
Li~Yang and Anastasios~A Tsiatis.
\newblock Efficiency study of estimators for a treatment effect in a pretest-posttest trial.
\newblock {\em The American statistician}, 55:314--321, 2001.

\bibitem{You2017}
Seungil You, David Ding, Kevin Canini, Jan Pfeifer, and Maya Gupta.
\newblock Deep lattice networks and partial monotonic functions.
\newblock In I~Guyon, U~Von Luxburg, S~Bengio, H~Wallach, R~Fergus, S~Vishwanathan, and R~Garnett, editors, {\em Advances in Neural Information Processing Systems}, volume~30. Curran Associates, Inc., 2017.

\bibitem{Min2006}
Min-Ling Zhang and Zhi-Hua Zhou.
\newblock Multilabel neural networks with applications to functional genomics and text categorization.
\newblock {\em IEEE Transactions on Knowledge and Data Engineering}, 18(10):1338--1351, 2006.

\bibitem{zhang2014reviewmlc}
Min-Ling Zhang and Zhi-Hua Zhou.
\newblock A review on multi-label learning algorithms.
\newblock {\em IEEE Transactions on Knowledge and Data Engineering}, 26(8):1819--1837, 2014.

\bibitem{zhang2018overview}
Yu~Zhang and Qiang Yang.
\newblock An overview of multi-task learning.
\newblock {\em National Science Review}, 5(1):30--43, 2018.

\bibitem{zhang2021survey}
Yu~Zhang and Qiang Yang.
\newblock A survey on multi-task learning.
\newblock {\em IEEE Transactions on Knowledge and Data Engineering}, 34(12):5586--5609, 2021.

\end{thebibliography}

\clearpage
\appendix

\setcounter{table}{0}
\renewcommand{\thetable}{A\arabic{table}}

\setcounter{figure}{0}
\renewcommand{\thefigure}{A\arabic{figure}}

\onecolumn
\section*{Appendix}
The Appendix is structured as follows.
Appendix \ref{ap:notation} summarizes the notations used in the main text and the Appendix.
Appendix \ref{ap:proof} provides the proofs of the theorems. Appendix \ref{app:multiplier-bootstrap} outlines the multiplier bootstrap procedure used for obtaining the confidence intervals in this paper.
Appendix \ref{ap:simulation} describes the simulation settings, Appendix \ref{ap:water_consumption} presents the additional details of the field experiment to reduce water consumption, and Appendix \ref{appendix:abema} offers the supplemental information about the content promotion experiment on ABEMA.

\section{Summary of Notation}
\label{ap:notation}
The notations used in the main text are summarized in Table \ref{tab:notation_summary}.

\begin{table}[ht]
    \centering
    \begin{tabular}{|l p{0.7\linewidth}|}
    \hline
    \textbf{Notation} & \textbf{Definition} \\
    \hline
    $X_i$ & pre-treatment covariates \\
    $W_i$ & treatment variable \\
    $Y_i$ & outcome variable \\
    $Z_i$ & observed samples, $Z_i := (X_i, W_i, Y_i)$ \\
    $Y_i(w)$ & potential outcome for treatment group $w$ \\
    $\tilde{\mathcal{Y}}$ & a set of scalar-valued locations to estimate the treatment effect for \\
    
    $\pi_w$ & treatment assignment probability for treatment group $w$ \\
    $n_w$ & number of observations in treatment group $w$ \\
    $n$ & number of observations \\
    $\widehat\pi_w$ & $n_w/n$, estimated treatment assignment probability for treatment group $w$ \\
    %$f_{Y(w)}(y)$ & potential outcome probability distribution function \\
    $F_{Y(w)}(y)$ & $E[\mathbf{1}_{\{Y(w) \leq y\}}]$, potential outcome distribution function \\
    $\gamma^{(w)}_y(x)$ & $E[\mathbf{1}_{\{Y(w)\leq y\}}|X=x]$, conditional distribution function \\
    $\widehat\gamma^{(w)}_y(x)$ & estimator of conditional distribution function $\gamma^{(w)}_y(x)$ \\
    $\mathbf{{\Gamma}}^{(w)}$ & a set of conditional distribution functions $\gamma^{(w)}_y(\cdot)$ for $y \in \tilde{\mathcal{Y}}$ \\
    $\mathbf{\widehat{\Gamma}}^{(w)}$ & estimator of $\mathbf{{\Gamma}}^{(w)}$ \\
\hline
\end{tabular}
\caption{Summary of Notation}
\label{tab:notation_summary}
\end{table}

\section{Proofs} \label{ap:proof}

\textbf{Proof of Theorem \ref{theorem:efficiency}} \\
We follow the approach used in \cite{Byambadalai2024} to prove the efficiency gain of the adjusted estimation. For completeness of the proof, we first present a variant of Lagrange's identity and Bergstr\"{o}m's inequality in the below lemma, which is useful to prove the efficiency gain of the regression adjustment. 

%% Lemma =====================================
\vspace{0.2cm}
\begin{lemma}
    \label{lemma:L-ind}
    For any 
    $(a_{1}, \dots, a_{K}) \in \mathbb{R}^{K}$
    and 
    $(b_{1}, \dots, b_{K}) \in \mathbb{R}^{K}$
    with $b_{k} > 0 $ for all $k=1, \dots, K$,
    we can show that 
    \begin{align*}
     \sum_{k=1}^{K} \frac{a_{k}^{2}}{b_{k}}   
     - 
     \frac{
     \big (
     \sum_{k=1}^{K} a_{k}
     \big )^2 
    }{
     \sum_{k=1}^{K} b_{k}
    }
    = 
     \frac{
     1
    }{
     \sum_{k=1}^{K} b_{k}
    }
    \cdot 
    \frac{1}{2}
    \sum_{k=1}^{K}
    \sum_{\substack{\ell =1 \\ \ell \neq k}}^K 
    \frac{
    (a_{k} b_{\ell} 
    - 
    a_{\ell} b_{k}
    )^2      
    }{
    b_{k} b_{\ell}
    } ,
    \end{align*}
    which implies 
    Bergstr\"{o}m's inequality, given by 
  \begin{align*}
     \sum_{k=1}^{K} \frac{a_{k}^{2}}{b_{k}}   
     \geq 
     \frac{
     \big (
     \sum_{k=1}^{K} a_{k}
     \big )^2 
    }{ 
     \sum_{k=1}^{K} b_{k}
    }
    . 
    \end{align*}

\end{lemma}
\begin{proof}
    Lagrange's identity is that, 
    for any 
    $(c_{1}, \dots, c_{K}) \in \mathbb{R}^{K}$
    and 
    $(d_{1}, \dots, d_{K}) \in \mathbb{R}^{K}$, 
    \begin{align}
    \label{eq:L-ind}
        \bigg (\sum_{k=1}^{K} c_{k}^2
        \bigg)
        \bigg(\sum_{k=1}^{K} d_{k}^2
        \bigg)
        - 
        \bigg(\sum_{k=1}^{K} c_{k}d_{k}
        \bigg)^2
        =
    \frac{1}{2}
    \sum_{k=1}^{K}
    \sum_{\substack{\ell =1 \\ \ell \neq k}}^K 
     (c_{k} d_{\ell} 
    - 
    c_{\ell} d_{k}
    )^2   .
    \end{align}
    Fix arbitrary 
    $(a_{1}, \dots, a_{K}) \in \mathbb{R}^{K}$
    and 
    $(b_{1}, \dots, b_{K}) \in \mathbb{R}^{K}$
    with $b_{k} > 0 $ for all $k=1, \dots, K$. 
    Then, 
    taking 
    $c_{k} = a_{k} / \sqrt{b_{k}}$ 
    and 
    $d_{k} = \sqrt{b_{k}}$
    for all $k=1, \dots, K$
    in (\ref{eq:L-ind}), we can show that 
    \begin{align*}
        \bigg (
        \sum_{k=1}^{K} 
        \frac{a_{k}^2}{b_{k}}
        \bigg)
        \bigg(\sum_{k=1}^{K} b_{k}
        \bigg)
        - 
        \bigg(\sum_{k=1}^{K} a_{k}
        \bigg)^2
     &   =
    \frac{1}{2}
    \sum_{k=1}^{K}
    \sum_{\substack{\ell =1 \\ \ell \neq k}}^K 
     \bigg (
     \frac{a_{k}}{ \sqrt{b_{k}}}
     \sqrt{b_{\ell}} 
     - 
     \frac{a_{\ell}}{ \sqrt{b_{\ell}}}
     \sqrt{ b_{k}} 
     \bigg )^2   \\ 
     & = 
         \frac{1}{2}
    \sum_{k=1}^{K}
    \sum_{\substack{\ell =1 \\ \ell \neq k}}^K 
    \frac{
    (a_{k} b_{\ell} 
    - 
    a_{\ell} b_{k}
    )^2      
    }{
    b_{k} b_{\ell}
    },
    \end{align*}
    which leads to the desired equality. 
    Also, the last expression in the math display above 
    is non-negative, which leads to Bergstr\"{o}m's inequality. 
\end{proof}
\vspace{0.5cm}
%%================================================================

To establish Theorem \ref{theorem:efficiency}, we first introduce additional notation. Specifically, we define the empirical probability measures of $X$ as 
\begin{align*}
    \widehat{\mathbb{F}}_{X} :=   
    \frac{1}{n}
    \sum_{i=1}^{n} 
    \delta_{X_i}
    \ \ \mathrm{and} \ \ 
    \widehat{\mathbb{F}}_{X}^{(w)} :=   
    \frac{1}{n_{w}}
    \sum_{i=1}^{n} 
    \mathbf{1}_{\{ W_{i} = w \} } \cdot 
    \delta_{X_i},
\end{align*}
for all observations and observations in the treatment group $w \in \mathcal{W}$, respectively. 
Here, $\delta_{x}$ is the measure that assigns mass 1 at $x \in \mathcal{X}$
and thus $\widehat{\mathbb{F}}_{X}$ and $\widehat{\mathbb{F}}_{X}^{(w)}$ can be interpreted as the random discrete
probability measures, which put mass $1/n$  and $1/n_{w}$ at each of the $n$ and $n_{w}$ points $\{X_{i}\}_{i=1}^{n}$
and 
$\{X_{i}:W_{i} =w\}_{i=1}^{n}$, respectively. 
Given a real-valued function $f: \mathcal{X} \to \mathbb{R}$, we denote by 
\begin{align*}
    \widehat{\mathbb{F}}_{X} f 
    = \int f d \widehat{\mathbb{F}}_{X} 
    = \frac{1}{n} \sum_{i=1}^{n}f(X_{i}) 
\end{align*}
and 
\begin{align*}
    \widehat{\mathbb{F}}_{X}^{(w)} f 
    = \int f d \widehat{\mathbb{F}}_{X}^{(w)} 
    = \frac{1}{n_{w}} \sum_{i=1}^{n} \mathbf{1}_{\{ W_{i} = w \} } \cdot f(X_{i}) . 
\end{align*}

Given that the true conditional distribution $\gamma_{y}^{(w)}(X) \equiv F_{Y(w)|X}(y|X)$, 
the infeasible version of regression-adjusted distribution function for treatment $w \in \mathcal{W}$ is written as
\begin{align*}
  \widetilde{\mathbf{F}}_{Y(w)}(y)
  = 
  \widehat{\mathbf{F}}_{Y(w)}^{empirical}(y)
  - 
  (\widehat{\mathbb{F}}_{X}^{(w)} - \widehat{\mathbb{F}}_{X})
  \gamma_{y}^{(w)}. 
\end{align*}

%%=====================================================================
%% proof
\vspace{0.5cm}
\begin{proof}[\textbf{Proof of Theorem \ref{theorem:efficiency}}] 
\textbf{Part (a)}
  Choose any arbitrary $w \in \mathcal{W}$ and $y \in \mathcal{Y}$.  
  Applying the quadratic expansion for 
  $
  \widetilde{\mathbf{F}}_{Y(w)}(y)
  = 
  \widehat{\mathbf{F}}_{Y(w)}^{empirical}(y)
  -
  (\widehat{\mathbb{F}}_{X}^{(w)} - \widehat{\mathbb{F}}_{X})  
  \gamma_{y}^{(w)} 
  $, 
  we can show that 
  \begin{align}
    \text{Var}
    \big(
      \widetilde{\mathbf{F}}_{Y(w)}(y)
    \big)
    =
    &
    \text{Var}
    \big(
    \widehat{\mathbf{F}}_{Y(w)}^{empirical}(y)
    \big)  
    -
    2
    \text{Cov}
    \Big (
    \widehat{\mathbf{F}}_{Y(1)}^{empirical}(y),
    (\widehat{\mathbb{F}}_{X}^{(w)} - \widehat{\mathbb{F}}_{X}) 
    \gamma_{y}^{(w)}
    \Big ) 
     +
    \text{Var}
    \Big (
    (\widehat{\mathbb{F}}_{X}^{(w)} - \widehat{\mathbb{F}}_{X}) 
    \gamma_{y}^{(w)}
    \Big ).     \label{eq:var-1}
  \end{align}

  We can write 
  $\widehat{\mathbb{F}}_{X} = \sum_{w' \in \mathcal{W}} 
  \hat{\pi}_{w'} \widehat{\mathbb{F}}_{X}^{(w')}
  $. 
  It is assumed that 
  observations are a random sample and $n_{w'} /n = \pi_{w'} + o(1)$ for every $w' \in \mathcal{W}$
  as $n \to \infty$. 
  Furthermore, all unconditional and conditional functions are 
  bounded. 
  By applying the dominated convergence theorem, we can show 
  \begin{align}
    n \text{Cov}
    \Big (
    \widehat{\mathbf{F}}_{Y(w)}^{empirical}(y),
    (\widehat{\mathbb{F}}_{X}^{(w)} - \widehat{\mathbb{F}}_{X}) 
    \gamma_{y}^{(w)}
    \Big ) \notag
    &
      = n
    \text{Cov}    \Big (
    \widehat{\mathbf{F}}_{Y(w)}^{empirical}(y),
    (1 - \hat{\pi}_{w})\widehat{\mathbb{F}}_{X}^{(w)} 
    \gamma_{y}^{(w)}(X)
    \Big ) \notag \\
    &
      = 
    \frac{1 - \pi_{w}}{ \pi_{w}}
    \text{Cov}
    \big (
      \mathbf{1}_{ \{Y(w) \leq y\} },
      \gamma_{y}^{(w)}(X)
    \big ) + o(1). 
    \label{eq:cov-1}
  \end{align}
   Similarly, we can show that  
  \begin{align}
    n \text{Var}
    \big (
    (\widehat{\mathbb{F}}_{X}^{(w)} - \widehat{\mathbb{F}}_{X}) 
    \gamma_{y}^{(w)}
    \big ) \notag 
    & 
      =
    n \text{Var}
    \big (
    (1 -   \hat{\pi}_{w})
    \widehat{\mathbb{F}}_{X}^{(w)} 
    \gamma_{y}^{(w)}
    \big )
    + 
    n
    \sum_{w' : w' \not = w }
     \text{Var}
    \big (
    \hat{\pi}_{w'}
    \widehat{\mathbb{F}}_{X}^{(w')} 
    \gamma_{y}^{(w)}
    \big )
 \notag \\ 
    & 
      =
 \frac{  (1 - \pi_{w})^2}{ \pi_{w}} 
    \text{Var}
    \big (
    \gamma_{y}^{(w)}(X)
    \big )
    +
      \sum_{w' : w' \not = w }
      \frac{    \pi_{w'}^{2}}{\pi_{w'}}
    \text{Var}
    \big (
    \gamma_{y}^{(w)}(X)
    \big ) + o(1)
    \notag \\ 
    & 
      =
      \bigg (
           \frac{  (1 -   \pi_{w})^2}{ \pi_{w}} 
        +
      \sum_{w' : w' \not = w }
      \pi_{w'} 
      \bigg )
    \text{Var}
    \big (
    \gamma_{y}^{(w)}(X)
    \big ) + o(1)\notag \\ 
    & 
  =
  \frac{1-\pi_{w}}{\pi_{w}}
    \text{Var}
    \big (
    \gamma_{y}^{(w)}(X)
    \big ) + o(1). 
    \label{eq:cov-2}
  \end{align}
  It follows from 
  (\ref{eq:var-1})-(\ref{eq:cov-2}) that
  \begin{align}
    \label{eq:bb2}
    n
    \big \{
    \text{Var} \big(
    \widehat{\mathbf{F}}_{Y(w)}^{empirical}(y)
    \big)
    -
    \text{Var} \big(
      \widetilde{\mathbf{F}}_{Y(w)}(y)
    \big)
      \big \}
    &= 
    \frac{1-\pi_{w}}{\pi_{w}}
    \big \{
     2
    \text{Cov}
    \big (
      \mathbf{1}_{ \{Y(w) \leq y\} },
      \gamma_{y}^{(w)}(X)
    \big )
    -
    \text{Var}
    \big (
    \gamma_{y}^{(w)}(X)
    \big )
    \big \} + o(1).
  \end{align}
    %% optimal estimator 
  An application of the law of iterated expectation yields \begin{align*}
      \text{Cov}
    \big (
      \mathbf{1}_{ \{Y(w) \leq y\} },
      \gamma_{y}^{(w)}(X)
    \big )
    =
    \text{Var} \big (    E[  \mathbf{1}_{ \{Y(w) \leq y\} }|X]     \big ),
  \end{align*}
  which together with (\ref{eq:bb2}) shows  
   \begin{align*}
    n
    \big \{
    \text{Var} \big(
    \widehat{\mathbf{F}}_{Y(w)}^{empirical}(y)
    \big)
    -
    \text{Var} \big(
      \widetilde{\mathbf{F}}_{Y(w)}(y)
    \big) 
      \big \}
    &= 
    \frac{1-\pi_{w}}{\pi_{w}}
    \text{Var}
      \big (
      \gamma_{y}^{(w)}(X)       
      \big ) + o(1). 
  \end{align*}
  Since $\pi_{w} \in (0, 1)$ and
  $    \text{Var}
      \big (
      \gamma_{y}^{(w)}(X)      
      \big )  \geq 0$, 
  it follows that 
  $
    \text{Var} \big(
    \widehat{\mathbf{F}}_{Y(w)}^{empirical}(y)
    \big)
    \geq 
      \text{Var} \big(
      \widetilde{\mathbf{F}}_{Y(w)}(y)
    \big) + o(n^{-1}).
  $
  Here, the equality hold 
  only when 
  $F_{Y(w)|X}(y) = F_{Y(w)}(y)$
  or  
  $X$ has no predictive power for the event $\mathbf{1}_{ \{Y(w) \leq y\} }$.  

  \vspace{0.2cm}
    \textbf{Part (b)}
    Choose any arbitrary $y \in \mathcal{Y}$.  
    First, we shall show that, for any $w, w' \in \mathcal{W}$, 
    \begin{align}
    \label{eq:cov-A}
        n
        \text{Cov} \big(
        \widehat{\mathbf{F}}_{Y(w)}(y), \widetilde{\mathbf{F}}_{Y(w')}(y) 
        \big)
        =
        \text{Cov} \big(
        \gamma_{y}^{(w)}(X),
        \gamma_{y}^{(w')}(X)
        \big). 
    \end{align}
    Fix any two distinct treatment statuses $w, w' \in \mathcal{W}$. 
    We can write 
  $
  \widetilde{\mathbf{F}}_{Y(w)}(y)
  = 
  \big (\widehat{\mathbf{F}}_{Y(w)}^{empirical}(y)
  - \widehat{\mathbb{F}}_{X}^{(w)}\gamma_{y}^{(w)}
  \big ) 
  +
  \widehat{\mathbb{F}}_{X} \gamma_{y}^{(w)} 
  $
  and also  
  $\widehat{\mathbb{F}}_{X}  \gamma_{y}^{(w)} = \sum_{v\in \mathcal{W}} 
  \hat{\pi}_{v} \widehat{\mathbb{F}}_{X}^{(v)}  \gamma_{y}^{(w)} 
  $.   
  Given random sample and the bi-linear property of the covariance function, 
  we can show that 
  \begin{align*}
        \text{Cov} \Big(
        \widetilde{\mathbf{F}}_{Y(w)}(y), 
        \widetilde{\mathbf{F}}_{Y(w')}(y) 
        \Big)
        =&
        \text{Cov}
        \Big(
        \widehat{\mathbf{F}}_{Y(w)}^{empirical}(y)
        - \widehat{\mathbb{F}}_{X}^{(w)} \gamma_{y}^{(w)} ,
        \hat{\pi}_{w'}\widehat{\mathbb{F}}_{X}^{(w')}  \gamma_{y}^{(w')} 
        \Big ) 
        \\ 
        & + \text{Cov}
        \big(
        \hat{\pi}_{w}\widehat{\mathbb{F}}_{X}^{(w)}  \gamma_{y}^{(w)}, 
        \widehat{\mathbf{F}}_{Y(w')}^{empirical}
        - \widehat{\mathbb{F}}_{X}^{(w')} \gamma_{y}^{(w')}
        \big ) \\
        & + \text{Cov}
        \big(
        \widehat{\mathbb{F}}_{X}  \gamma_{y}^{(w)}, 
        \widehat{\mathbb{F}}_{X}  \gamma_{y}^{(w')}
        \big ) 
    , 
    \end{align*}
    where it can be shown that the first and second terms on the right-hand side are equal zero, due to the fact that
    $     E \big [   \widehat{\mathbf{F}}_{Y(w)}^{empirical}(y)
        - \widehat{\mathbb{F}}_{X}^{(w)}\gamma_{y}^{(w)}|X_{1}, \dots, X_{n}
        \big ] = 0
    $. 
    Furthermore, under the random sample assumption, we can show that 
    \begin{align*}
    \text{Cov}
        \big(
        \widehat{\mathbb{F}}_{X}  \gamma_{y}^{(w)}, 
        \widehat{\mathbb{F}}_{X}  \gamma_{y}^{(w')}
        \big ) =
        n^{-1}
      \text{Cov}
        \big(
         \gamma_{y}^{(w)}(X), 
         \gamma_{y}^{(w')}(X)
        \big )
        \end{align*}. 
        Thus, we can prove the equality in (\ref{eq:cov-A}).

    % variance- covariance 
    Next, we compare the variance-covariance matrices of the simple and regression-adjusted estimators.
    By applying the result from part (a) of this theorem and the one in (\ref{eq:cov-A}), we are able to show that
    \begin{align*}
    & n 
    \big \{
     \text{Var} \big (
    \widehat{\theta}_{y}^{empirical}
    \big ) 
    -
    \text{Var} \big (
    \widetilde{\theta}_{y}
    \big )  
    \big \} \\ 
    & \ \ \ \ \ \ 
    = 
    \left [
    \begin{array}{cccc}
     \frac{1-\pi_{1}}{\pi_{1}}
    \text{Var}
      \big (
      \gamma_{y}^{(1)}(X)       
      \big ),
    & 
    -\text{Cov}
        \big(
         \gamma_{y}^{(1)}(X), 
         \gamma_{y}^{(2)}(X)
        \big ),   
        & 
        \dots,  
        & 
        -\text{Cov}
        \big(
         \gamma_{y}^{(1)}(X), 
         \gamma_{y}^{(K)}(X)
        \big )  \\
        -\text{Cov}
        \big(
         \gamma_{y}^{(2)}(X), 
         \gamma_{y}^{(1)}(X)
        \big ),   
        &     
        \frac{1-\pi_{2}}{\pi_{2}}
    \text{Var}
      \big (
      \gamma_{y}^{(2)}(X)    
      \big ), 
        & \dots,  
        & 
        -\text{Cov}
        \big(
         F_{Y(2)|X}, 
         F_{Y(K)|X}
        \big ) 
        \\ 
          \vdots &\vdots&\ddots & \vdots\\ 
        -\text{Cov}
        \big(
         \gamma_{y}^{(K)}(X), 
         \gamma_{y}^{(1)}(X)
        \big ),   
        &     
        -\text{Cov}
        \big(
         \gamma_{y}^{(K)}(X), 
         \gamma_{y}^{(2)}(X)
        \big ), 
        & \dots,  
        & 
        \frac{1-\pi_{K}}{\pi_{K}}
    \text{Var}
      \big (
      \gamma_{y}^{(K)}(X)       
      \big )
    \end{array}
    \right ]  + o(1),
 \end{align*}
which can be written as  
\begin{align*}
       n 
    \big \{
     \text{Var} \big (
    \widehat{\theta}_{y}^{empirical}
    \big ) 
    -
    \text{Var} \big (
    \widetilde{\theta}_{y}
    \big )  
    \big \}
    = 
    E 
    \Big [
     \big (\gamma_{y}(X) - E[\gamma_{y}(X)]\big) 
     A
     \big(\gamma_{y}(X) - E[\gamma_{y}(X)]\big)^{\top} 
    \Big ] + o(1),
\end{align*}
where 
$\gamma_{y}(X) = [\gamma_{y}^{(1)}(X), \dots, \gamma_{y}^{(K)}(X)]^{\top}$ and  
\begin{align*}
    A:= 
        \left [
    \begin{array}{cccc}
    {\pi}_{1}^{-1} - 1,
    & 
    -1,   
        & 
        \dots,  
        & 
        -1 \\
        -1,   
        &     
    {\pi}_{2}^{-1} - 1,
        & \dots,  
        & 
        -1 
        \\ 
          \vdots &\vdots&\ddots & \vdots\\ 
        -1,   
        &     
        -1, 
        & \dots,  
        & 
    {\pi}_{K}^{-1} - 1
    \end{array}
    \right ] .
    \end{align*}
    The variant of 
    Lagrange's identity
    in 
    Lemma \ref{lemma:L-ind}
    with $\sum_{w \in \mathcal{W}} \pi_{w} =1$
    shows that, 
    for an arbitrary vector $v:=(v_{1}, \dots, v_{K})^{\top} \in \mathbb{R}^{k}$,  
    \begin{align*}
     & v^{\top}
     \big (\gamma_{y}(X) - E[\gamma_{y}(X)]\big) 
     A
     \big(\gamma_{y}(X) - E[\gamma_{y}(X)]\big)^{\top} 
     v \\ 
      & \ \ = 
     \sum_{w \in \mathcal{W}}
     \frac{
     v_{w}^{2}
     \big ( 
     \gamma_{y}^{(w)}(X)
     -
     E[\gamma_{y}^{(w)}(X)]
     \big )^{2}
     }{
        \pi_{w}
     } 
      - 
     \bigg (
     \sum_{w \in \mathcal{W}}
     v_{w}
     \big ( 
     \gamma_{y}^{(w)}(X)
     -
     E[\gamma_{y}^{(w)}(X)]
     \big )
     \bigg )^{2} \\ 
     & \ \ \  = 
     \frac{1}{2}
    \sum_{w \in \mathcal{W}} 
    \sum_{\substack{w' \in \mathcal{W} \\ w' \neq w}}
     \frac{
     \big \{ 
         v_{w}
         \big ( 
         \gamma_{y}^{(w)}(X)
         -
         E[\gamma_{y}^{(w)}(X)]
         \big )
         {\pi}_{w'}
         -
         v_{w'}
         \big ( 
         \gamma_{y}^{(w')}(X)
         -
         E[\gamma_{y}^{(w')}(X)]
         \big )
         {\pi}_{w}     \big \} ^{2}
     }{
     {\pi}_{w}
     {\pi}_{w'}
     } . 
    \end{align*}
    It follows that 
    \begin{align*} 
       v^{\top}
    \big \{
     \text{Var} \big (
    \widehat{\theta}_{y}^{empirical}
    \big ) 
    -
    \text{Var} \big (
    \widetilde{\theta}_{y}
    \big )  
    \big \}
    v
    = 
    \frac{1}{2}
    \sum_{w \in \mathcal{W}} 
    \sum_{\substack{w' \in \mathcal{W} \\ w' \neq w}}
     \frac{
     \text{Var} 
     \Big (
         v_{w}
         \gamma_{y}^{(w)}(X)
         {\pi}_{w'}
         -
         v_{w'}
         \gamma_{y}^{(w')}(X)
         {\pi}_{w}     
         \Big )
     }{
     {\pi}_{w}
     {\pi}_{w'}
     } 
    + o(n^{-1}). 
    \end{align*}
  The above equality implies the desired positive semi-definiteness result, because $ \text{Var} 
     \big (
         v_{w}
         \gamma_{y}^{(w)}(X)
         {\pi}_{w'}
         -
         v_{w'}
         \gamma_{y}^{(w')}(X)
         {\pi}_{w}     
         \big )
         \geq  0$
         for any 
         $w, w' \in \mathcal{W}$ with $ w \neq w'$.
  
  Furthermore, the positive definite result holds when 
    $ \text{Var} 
     \big (
         v_{w}
         \gamma_{y}^{(w)}(X)
         {\pi}_{w'}
         -
         v_{w'}
         \gamma_{y}^{(w')}(X)
         {\pi}_{w}     
         \big )
         > 0$
         for any $v \in \mathbb{R}^{k}$ with $v \neq 0$
         and 
         for any 
         $w, w' \in \mathcal{W}$ with $ w \neq w'$. 
         Because 
        $v \in \mathbb{R}^{k}$ is chosen arbitrarily except  $v \neq 0$
        and $\pi_{w} \in (0,1)$ for all $w \in \mathcal{W}$, 
        the condition for the positive definiteness can be written as  
         $
             \text{Var} 
         \big (
         \gamma_{y}^{(w)}(X)
         -
         r
         \cdot 
         \gamma_{y}^{(w')}(X)
         \big )
         > 0
         $
         for any  $r \in \mathbb{R}$
         and 
         for any 
         $w, w' \in \mathcal{W}$ with $ w \neq w'$. 
\end{proof}

\section{Multiplier Bootstrap Procedure} \label{app:multiplier-bootstrap} 
We obtain pointwise confidence bands for distributional parameters, which are functionals of distribution functions, using multiplier bootstrap following \cite{chernozhukov2013inference} and \cite{belloni2017program}. We outline the procedure to obtain pointwise confidence bands in Algorithm \ref{alg:uniform-band} where $\phi((\theta_y)_{y\in\mathcal Y})$ is some functional of $(\theta_y)_{y\in\mathcal Y}$.

As a prerequisite, letting $\pi:=(\pi_1, \dots, \pi_K)^{\top}$, define moment functions
\begin{align*}
\psi_y(Z; \theta_y, \gamma_y, \pi):= 
    \big (\psi_y^{(1)}(Z; \theta_y, \gamma_y, \pi_1), \dots, 
    \psi_y^{(K)}(Z; \theta_y, \gamma_y, \pi_K) \big) ^{\top},    
\end{align*}
where, for each $w \in \mathcal{W}$,  
\begin{align} 
\label{eq:phi-def}
    \psi_y^{(w)}(Z; \theta_y, \gamma_y, \pi_w)
     := & 
    \frac{\mathbf{1}_{\{W=w\}} \cdot (\mathbf{1}_{\{Y\leq y\}}-\gamma_{y}^{(w)}(X))}{\pi_{w}} + \gamma_y^{(w)}(X) - \theta_y^{(w)}.
\end{align}

\begin{algorithm}[!h]
   \caption{Multiplier bootstrap procedure to obtain pointwise confidence bands}
   \label{alg:uniform-band}
\begin{algorithmic}
\State
\State {\bfseries Input:} Data $\{(X_i, W_i, Y_i)\}_{i=1}^{n}$; point estimates $\hat\theta_y$; influence functions $\hat \psi_y(Z_i):=\psi_y(Z_i; \hat\theta_y, \hat\gamma_y, \hat\pi)$

\State
   \begin{enumerate}
    \item[(1)] Draw multipliers $\{\xi_i\}_{i=1}^{n}= \{m_{1,i}/\sqrt{2} + ((m_{2,i})^2-1)/2\}_{i=1}^{n}$ independently from the data $\{Z_i\}_{i=1}^{n}$, where $m_{1,i}$ and $m_{2,i}$ are i.i.d. draws from two independent standard normal random variables.
    \item[(2)] For each $y\in\mathcal Y$, obtain the bootstrap draws $\phi^b(\hat\theta_y)$ of $\phi(\hat\theta_y)$ as 
    \begin{equation*}
    \phi^b(\hat\theta_y) =\phi(\hat\theta_y^b) \text{ where } \hat\theta_y^b=\hat\theta_y + \frac{1}{n}\sum_{i=1}^{n} \xi_i \hat \psi_y(Z_i).
    \end{equation*}
    \item[(3)] Repeat (1)-(2) $B$ times and index the bootstrap draws by $b=1, \dots, B$.
    \item[(4)] Obtain bootstrap standard error estimates for $\phi(\hat\theta_y)$ for each $y\in\mathcal Y$ as 
    \begin{equation*}
        \hat \Sigma(y) = \sum_{b=1}^{B}\frac{(\phi^b(\hat\theta_y) - \bar\phi(\hat\theta_y))^2}{B-1},
    \end{equation*}
    where $\bar\phi(\hat\theta_y)=\sum_{b=1}^{B}\frac{\phi^b(\hat\theta_y)}{B}$.
    \item[(5)] Construct $(1-\alpha)\times 100\%$ pointwise confidence band for $\phi((\theta_y)_{y\in\mathcal Y})$ as 
    \begin{equation*}
        I^{1-\alpha} := \{[\phi(\hat\theta_y) \pm \hat z_{1-\alpha} \times \hat\Sigma(y)]: y\in\mathcal Y\}.
    \end{equation*} 
\end{enumerate}
%\ENDFOR

\State {\bfseries Result:} $(1-\alpha)\times 100\%$ confidence band $I^{1-\alpha}$ for $\phi((\theta_y)_{y\in\mathcal Y})$
\end{algorithmic}
\end{algorithm}

\newpage
\section{Simulation Study}
\label{ap:simulation}
Here, we describe the parameters used in the simulation experiment and present additional experimental results not included in the main text. The first subsection describes the data generating process, the second subsection states the parameter details of the ML models and experiment environment, and the third subsection shows the numerical results of the simulation experiment.

\subsection{Data Generating Process (DGP)} 
\label{simulation:DGP1}
For the simulation experiment, we used the following DGP. We fix the number of covariates \( d_x \) as \( d_x = 20 \) and the sample size \( n \) to be \(  n = 1000 \). For each \( i = 1, \ldots, n \), we generate \( X_i = (X_{1i}, \ldots, X_{20i}) \) from \( U_{20}((0, 1)^{20}) \), a multivariate uniform distribution on \( (0, 1) \). Binary treatment variable \( W_i \) follows a Bernoulli distribution with a success probability of \( \rho = 0.5 \). A continuous outcome variable \( Y_i \) is then generated from the outcome equation \( Y_i = f(X_i, W_i) + U_i \), where the error term \( U_i \sim N(0, 1) \). We consider the functional form of

\begin{equation}
f(X_i, W_i) = \sum_{j=1}^{20} \sum_{k=1}^{20} \beta_j \beta_k X_{ji} X_{ki}
\end{equation}

so that the outcome includes the interactions of covariates. For coefficients $\beta_j$, we set

\[
\beta_j = \begin{cases} 
1 & \text{for } j \in \{1, \ldots, 18\} \\ 
W_i & \text{for } j \in \{19, 20\} 
\end{cases}
\]

Because $\beta_{19}$ and $\beta_{20}$ depend on the treatment variable, the records with $W_i = 1$ are more likely to take higher outcome values.
We used quantiles $q \in \{0.05, 0.1, \dots, 0.95\}$ for the locations, and observed a negative DTE across all of them. The absolute size of the DTE takes the maximum around the median of the outcome distribution.

\subsection{Model Implementation and Experiment Environment} 
\label{simulation:model}
The model used for the simulation study follows the structure in Figure \ref{fig:arch_model} with the following parameters in Table \ref{table:params_models}. Here, $h_i$ denotes the number of neurons in the $i$th hidden layer, Optimizer is the optimization algorithm for training neural networks, and Folds $L$ is the number of folds used for cross-fitting. The models are trained with binary cross-entry loss and Adam Optimizer. We implemented the experiment in Python and used the PyTorch \cite{paszke2019pytorchimperativestylehighperformance} framework to build the models. Experiments are run on a Macbook Pro with 36 GB memory and the Apple M3 Pro chip. The multi-task regression adjustment algorithm is available in the Python library \texttt{dte-adj} (\href{https://pypi.org/project/dte-adj/}{https://pypi.org/project/dte-adj/}). All models are trained on the CPU, and the same environment was used for all experiments. 

\begin{table}[ht!]
\centering
\begin{tabular}{|lccc|}
\hline
\textbf{Parameter} & \textbf{Simulation} & \textbf{Water Consumption} & \textbf{ABEMA} \\
\hline
Input size $d_x$ & 20 & 12 & 10 \\
Number of hidden layers & 3 & 3 & 3 \\
$h_1$ & 128 & 128 & 16 \\
$h_2$ & 64 & 64 & 16 \\
$h_3$ & 19 & 200 & 51 \\
$g$ & $exp$ & $exp$ & $exp$ \\ 
$f(x)$ & $arctan(x) / (\pi/2)$ & $\frac{1-exp(-x)}{1+exp(-x)}$ & $arctan(x) / (\pi/2)$ \\
$\sigma$ & ReLU & ReLU & ReLU \\
Learning Rate & 0.01 & 0.001 & 0.001 \\
Batch Size & 16 & 64 & 128 \\
Folds $L$ & 2 & 2 & 2 \\
\hline
\end{tabular}
\caption{Model Parameters for Empirical Studies} 
\begin{minipage}{0.5\textwidth}
   \textit{Notes:} All neural network models use identical parameters, including the number of folds and batch size except for the size of the last hidden layer.
\end{minipage}
\label{table:params_models}
\end{table}

\subsection{Pointwise Estimation Result}
The pointwise MSE reduction (\%) is summarized in Table \ref{tab:simulation_performance} and its raw values per location are available in Table \ref{tab:simulation_pointwise_performance}. In Table \ref{tab:simulation_performance}, the minimum, 25 percentile, median, 75 percentile and the maximum of pointwise SE reduction are reported. As shown in the result, the multi-task NN with the monotonic constraint achieves the highest MSE reduction in most locations, while the multi-task NN also outperforms the other two methods. We also applied BART \cite{Chipman_2010}, Wu et al. \cite{dnet}, and DNet \cite{jiang2023regression} to our DGP and compared our DTE MSE reduction (\%) with their QTE MSE reduction (\%) as reported below. BART and DNet are not designed to estimate unconditional QTE precisely, while Jiang et al. proposed a regression-adjustment method like ours, designed to decrease the variance of QTE estimations compared to unadjusted QTE. The result shows that the proposed multi-task NN adjustment achieves a higher variance reduction than BART and Jiang et al. and improves the variance across all quantiles, unlike DNet.

\begin{table}[ht!]
\centering
\begin{tabular}{|l|c|c|c|c|c|}
\hline
\textbf{Method} & \textbf{min} & \textbf{p25} & \textbf{p50} & \textbf{p75} & \textbf{max} \\ \hline
Linear Regression & 18.2 & 32.7 & 45.5 & 49.0 & 51.7 \\
Single-Task NN & 12.2 & 33.6 & 38.7 & 42.6 & 46.7 \\
Multi-Task NN & 48.6 & 58.3 & 59.7 & 62.3 & 65.5 \\
Monotonic Multi-Task NN & \textbf{50.1} & \textbf{60.7} & \textbf{62.1} & \textbf{64.6} & \textbf{67.5} \\ \hline
\end{tabular}
\caption{MSE Reduction in DTE Estimation: Simulation Study} 
\begin{minipage}{0.5\textwidth}
   \textit{Notes:} Percentage reduction in mean squared errors compared to empirical DTE across models. Sample size n=1,000 with S=500 simulation iterations.
\end{minipage}
\label{tab:simulation_performance}
\end{table}

\begin{table*}[h!]
    \centering
    \begin{tabular}{|c|p{2cm}|p{2cm}|p{2cm}|p{2cm}|p{1.5cm}|p{1.5cm}|p{1.5cm}|}
        \hline
        Quantile & Linear Regression (DTE) & Single-Task NN (DTE) & Multi-Task NN (DTE) & Monotonic Multi-Task NN (DTE) & DNet (QTE) & Jiang et al. (QTE) & BART (QTE) \\
        \hline
        0.05 & 18.20 & 16.45 & 51.38 & 50.05 & -449.32 & 18.48 & -35.44 \\
        0.10 & 27.99 & 32.49 & 57.99 & 58.23 & -290.99 & 29.02 & 0.17 \\
        0.15 & 32.36 & 34.66 & 59.79 & 63.38 & -229.61 & 36.59 & -2.01 \\
        0.20 & 36.41 & 39.81 & 59.66 & 62.86 & -143.00 & 37.28 & 11.80 \\
        0.25 & 38.51 & 39.52 & 60.61 & 60.88 & -85.91 & 45.38 & 6.25 \\
        0.30 & 45.53 & 45.91 & 65.14 & 64.92 & -28.17 & 51.06 & 12.83 \\
        0.35 & 45.65 & 46.40 & 65.52 & 66.27 & 38.71 & 50.54 & 10.69 \\
        0.40 & 46.54 & 46.66 & 65.27 & 66.87 & 68.05 & 52.97 & 6.42 \\
        0.45 & 48.10 & 43.06 & 63.08 & 65.97 & 82.56 & 52.71 & 11.06 \\
        0.50 & 51.72 & 43.97 & 64.48 & 67.48 & 89.63 & 55.91 & 7.48 \\
        0.55 & 50.73 & 42.05 & 61.61 & 64.36 & 81.49 & 57.83 & 11.81 \\
        0.60 & 49.91 & 38.69 & 58.28 & 62.20 & 61.27 & 53.70 & 6.03 \\
        0.65 & 50.53 & 41.65 & 58.86 & 60.99 & 32.96 & 54.96 & 13.17 \\
        0.70 & 50.70 & 38.59 & 58.63 & 61.97 & 10.43 & 51.11 & 8.94 \\
        0.75 & 46.45 & 38.25 & 60.60 & 62.88 & -44.92 & 47.14 & 12.72 \\
        0.80 & 38.19 & 35.67 & 58.78 & 60.49 & -80.15 & 45.94 & 5.94 \\
        0.85 & 33.09 & 27.68 & 58.33 & 60.86 & -159.48 & 36.40 & 11.59 \\
        0.90 & 29.32 & 24.91 & 51.38 & 56.10 & -174.35 & 35.02 & 8.00 \\
        0.95 & 20.71 & 12.23 & 48.64 & 52.64 & -154.77 & 24.29 & -2.17 \\
        \hline
    \end{tabular}
    \caption{Pointwise MSE reduction (\%) relative to the empirical DTE or QTE across different estimation methods in the simulation study. n=1,000, S=500.}
    \label{tab:simulation_pointwise_performance}
\end{table*}

\begin{table}[h]
\centering
\begin{tabular}{|l|c|c|c|c|c|}
\hline
\textbf{Model} & \textbf{Accuracy ↑} & \textbf{Precision ↑} & \textbf{Recall ↑} \\
\hline
Linear Regression & 0.933 & 0.937 & 0.928   \\ 
Single-Task NN & 0.891 & 0.889 & 0.892  \\ 
Multi-Task NN & 0.954 & 0.946 & \textbf{0.963}  \\ 
Monotonic Multi-Task NN & \textbf{0.959} & \textbf{0.970} & 0.945  \\ 
\hline
\end{tabular}
\caption{Comparison of prediction performance on the simulation data.}
\label{table:ablation_simulation_performance}
\end{table}

\clearpage
\section{Nudges to Reduce Water Consumption}
\label{ap:water_consumption}
Here, we describe the parameters used in Nudges to Reduce Water Consumption experiment and present additional experimental results that were not included in the main text.
The dataset from the randomized experiment can be downloaded at https://doi.org/10.7910/DVN1/22633 \cite{Ferraro2013}. The covariates used for this experiment are monthly water consumption during the year prior to the experiment. The locations used for this experiment are $\tilde{\mathcal{Y}} = (1, \dots, 199, 200) ^ \mathrm{T}$ and Figure~\ref{fig:water_consumption_hist} shows the histogram of outcomes. As described in the main text, most records are distributed from 0 to 100 and the range 100 to 200 contains a limited number of records.
The parameters used for this experiment are described in Table \ref{table:params_models}. The numerical results in Table \ref{tab:water_consumption_se_reduction} align with the simulation experiment, showing that the multi-task neural network outperforms the other two methods, with the monotonic constraint further enhancing precision.

\begin{figure}[!h]
    \centering
    \includegraphics[width=0.45\linewidth]{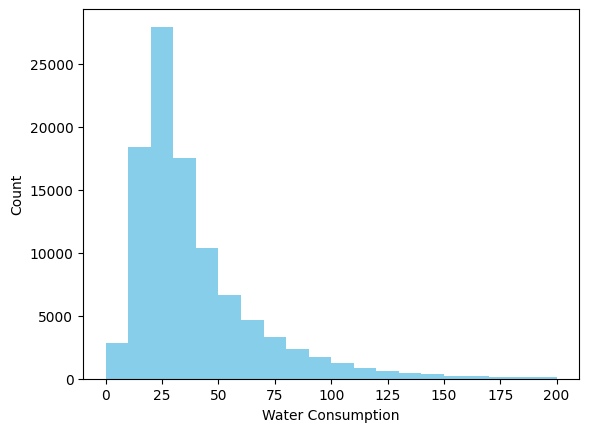}
    \caption{Water Consumption Distribution} 
%\vspace{-0.3cm}
\begin{minipage}{0.5\textwidth}
    \textit{Notes:} Values in thousands of gallons. Most observations range between 0 and 100 thousand gallons.
\end{minipage}
    \label{fig:water_consumption_hist}
\end{figure}

\begin{table}[h!]
\centering
\begin{tabular}{|l|c|c|c|c|c|}
\hline
\textbf{Method} & \textbf{min} & \textbf{p25} & \textbf{p50} & \textbf{p75} & \textbf{max} \\ \hline
Linear Regression & -0.35 & 7.39 & 12.6 & 20.5 & 26.5 \\
Single-Task NN & -2.11 & 12.6 & 18.2 & 28.7 & 33.0 \\
Multi-Task NN & \textbf{6.78} & 17.4 & 19.1 & 29.5 & 32.9 \\
Monotonic Multi-Task NN & 3.08 & \textbf{17.5} & \textbf{20.3} & \textbf{29.5} & \textbf{33.3} \\ \hline
\end{tabular}
\caption{Summary of SE Reduction in DTE Estimation (\%): Water Conservation Experiment} 
\begin{minipage}{0.7\textwidth}
   \textit{Notes:} Percentage reduction in standard errors compared to the empirical DTE across models. Sample size n=78,500 with B=5,000 multiplier bootstrap iterations.
\end{minipage}
\label{tab:water_consumption_se_reduction}
\end{table}

% \clearpage
\section{ABEMA Content Promotion}
\label{appendix:abema}
Here, we describe the parameters used in the ABEMA content promotion experiment and present additional experimental results that were not included in the main text. The locations used for this experiment are $\tilde{\mathcal{Y}} = (0, 1, \dots, 50) ^ \mathrm{T}$. 
Table~\ref{tab:abema_performance} presents the numerical values of the MSE reduction rate. As in the previous two experiments, the multi-task NN and monotonic multi-task NN adjustments outperform the other two adjustment methods. Unlike the other two experiments, no clear improvement was seen by the monotonicity constraint. We hypothesize that the large sample size in this experiment eliminates the need for neural network models to benefit from the monotonic constraint. 
The parameters used for this experiment are described in Table \ref{table:params_models}. We limit the size of neurons in the NN in this experiment due to its large sample size.

\begin{table}[h!]
\centering
\begin{tabular}{|l|c|c|c|c|c|}
\hline
\textbf{Method} & \textbf{min} & \textbf{p25} & \textbf{p50} & \textbf{p75} & \textbf{max} \\ \hline
Linear Regression & \textbf{0.406} & 5.10 & 6.31 & 7.16 & 14.3 \\ 
Single-Task NN & 0.249 & 5.49 & 6.45 & 7.86 & 15.3 \\ 
Multi-Task NN & 0.305 & \textbf{5.65} & \textbf{7.02} & \textbf{7.95} & \textbf{15.7} \\ 
Monotonic Multi-Task NN & 0.310 & \textbf{5.65} & 7.01 & \textbf{7.95} & \textbf{15.7} \\ \hline
\end{tabular}
\caption{Summary of SE Reduction in DTE Estimation (\%): ABEMA Content Promotion} 
\begin{minipage}{0.5\textwidth}
    \textit{Notes:} SE reduction (\%) is calculated over the empirical DTE across models. Sample size n=4,311,905 with B=5,000 multiplier bootstrap iterations.
\end{minipage}
\label{tab:abema_performance}
\end{table}

\section{Validity of Sub-linear Assumption}
\label{appendix:sub-linear}

\subsection{Computational Analysis of Linear Regression}

Consider the following notation:
\begin{itemize}
    \item $X \in \mathbb{R}^{n \times d}$: input matrix
    \item $Y \in \mathbb{R}^{n \times p}$: output matrix with $p$ target dimensions  
    \item $B \in \mathbb{R}^{d \times p}$: parameter matrix
\end{itemize}
We assume $n \gg d, p$.

\vspace{0.5cm}
\textbf{Joint Multivariate Linear Regression.} 
The closed-form solution for joint multivariate linear regression is given by
\begin{align*}
% \label{eq:joint-solution}
B = (X^\top X)^{-1} X^\top Y .
\end{align*}

The computational complexity consists of the following components:
\begin{itemize}
    \item $\mathcal{O}(nd^2)$ for computing $X^\top X$
    \item $\mathcal{O}(d^3)$ for matrix inversion  
    \item $\mathcal{O}(ndp)$ for computing $X^\top Y$
    \item $\mathcal{O}(d^2p)$ for the final matrix multiplication
\end{itemize}

Therefore, the total computational cost of joint multivariate linear regression is 
\begin{align}
\label{eq:joint-complexity}
\mathcal{O}(nd^2 + d^3 + ndp + d^2p) .
\end{align}

\textbf{Separate Univariate Linear Regression.} 
For training univariate linear regression separately, 
we compute the regression coefficient for each output dimension $y^{(i)} \in \mathbb{R}^n$:
\begin{align*}
% \label{eq:separate-solution}
\beta^{(i)} = (X^\top X)^{-1} X^\top y^{(i)} .
\end{align*}

The computational cost per output dimension $y^{(i)}$ is $\mathcal{O}(nd^2 + d^3)$. 
Consequently, the total computational cost over all $p$ outputs is 
\begin{align}
\label{eq:separate-complexity}
\mathcal{O}(pnd^2 + pd^3) .
\end{align}

\textbf{Efficiency Comparison.} 
Given that $n \gg d, p$, we observe that
\begin{align}
\label{eq:complexity-comparison}
\mathcal{O}(nd^2 + d^3 + ndp + d^2p) \ll \mathcal{O}(pnd^2 + pd^3) ,
\end{align}
which demonstrates that linear regression satisfies the sub-linear computational scaling property outlined in Assumption 3, 
confirming that joint estimation provides significant computational advantages over separate univariate approaches.

\subsection{Empirical Computation Complexity Analysis of Neural Networks}
We hypothesize that deeper architectures, such as ResNet can achieve substantially higher computational cost reduction under our proposed framework. To empirically validate this hypothesis, we analyze the computational complexity of several neural network architectures by measuring multiply-accumulate (MAC) operations for models with single versus multiple outputs using the \texttt{ptflop} Python library.

\begin{table}[h!]
    \centering
    \begin{tabular}{|l|c|c|}
        \hline
        Model & 1 output & 20 outputs \\
        \hline
        Our Dense NN & 11.2 K & 12.4 K \\
        BERT & 852.23 M & 852.24 M \\
        ResNet-18 & 1.82 G & 1.82 G \\
        \hline
    \end{tabular}
    \caption{Computational complexity comparison across different neural network architectures measured in MAC operations}
    \label{tab:computational_complexity}
\end{table}

These results demonstrate that complex models satisfy the sub-linear scaling property outlined in Assumption \ref{assumption:computation}. Notably, the computational complexity of both BERT and ResNet-18 remains virtually unchanged when scaling from single to multiple outputs, confirming near-constant computational overhead. This behavior arises because modern deep architectures such as convolutional neural networks (CNNs) and Transformers typically employ deep encoding layers that extract shared representations, making them particularly well-suited for our multi-task learning approach.

\end{document}